\documentclass[10pt,final]{article}

\pdfoutput=1 

\usepackage[%
		minnames=1,maxnames=99,maxbibnames=99,minalphanames=1,maxalphanames=99,
		style=alphabetic,
		sorting=nyt,
		sortcites=false, %
		doi=false,url=false,
		uniquename=false,
		giveninits=true,
		natbib,
		backend=biber]{biblatex}

\usepackage[utf8]{inputenc}
\usepackage[T1]{fontenc}

\usepackage{lmodern} \normalfont %
\usepackage{anyfontsize}
\DeclareFontShape{T1}{lmr}{bx}{sc} { <-> ssub * cmr/bx/sc }{}
\DeclareFontShape{T1}{lmr}{m}{scit}{ <-> ssub * cmr/m/sc }{}
\DeclareFontShape{T1}{lmr}{bx}{scit}{ <-> ssub * cmr/bx/sc }{}

\usepackage{hyperref}
\hypersetup{
        bookmarks=false,
	linktocpage=true,
	colorlinks=true,				
	linkcolor=DarkBlue,				
	citecolor=DarkBlue,				
	urlcolor=DarkBlue,
}
\usepackage{etoolbox}

\usepackage{url}            %
\usepackage{amsfonts}       %
\usepackage{microtype}
\usepackage{graphicx}
\usepackage{booktabs} %
\usepackage{subcaption}
\usepackage{lipsum}
\usepackage{makecell}
\usepackage{arydshln}
\usepackage{wrapfig}
\pdfoutput=1
\usepackage{amsmath,amsthm,wrapfig}
	\usepackage[utf8]{inputenc} 
\usepackage{amsmath, amssymb,bm, cases, mathtools, thmtools}
\usepackage{verbatim}
\usepackage{graphicx}\graphicspath{{figures/}}
\usepackage{multicol}
\usepackage{tabularx}
\usepackage{mathrsfs} 
\usepackage{caption}
\usepackage{algorithm}
\usepackage{algorithmicx}
\usepackage[noend]{algpseudocode}
\usepackage{array,booktabs,arydshln,xcolor}
\definecolor{DarkBlue}{rgb}{0.2,0.2,0.6}

\renewbibmacro{in:}{%
	\ifentrytype{article}{}{\printtext{\bibstring{in}\intitlepunct}}}

\usepackage[T1]{fontenc}
\usepackage{inconsolata}
\usepackage{dsfont}
\usepackage{listings} %

\usepackage{enumitem}
\usepackage{booktabs}       %
\usepackage{nicefrac}       %

\usepackage{lineno}
\usepackage{bbm}

\usepackage{caption}
\usepackage{subcaption}

\usepackage{datetime}

\DeclareMathAlphabet\EuRoman{U}{eur}{m}{n}
\SetMathAlphabet\EuRoman{bold}{U}{eur}{b}{n}

\declaretheorem[style=plain,numberwithin=section,name=Theorem]{theorem}
\declaretheorem[style=plain,sibling=theorem,name=Lemma]{lemma}

\declaretheorem[style=plain,sibling=theorem,name=Corollary]{corollary}

\declaretheorem[style=definition,sibling=theorem,name=Definition]{definition}

\declaretheorem[style=remark,qed=$\triangleleft$,sibling=theorem,name=Remark]{remark}
\numberwithin{theorem}{section}

\usepackage{xparse}
\usepackage{xstring}
\usepackage{xspace}

\usepackage[color=green!25,prependcaption,textsize=tiny]{todonotes}

\def\[#1\]{\begin{align}#1\end{align}}
\def\*[#1\]{\begin{align*}#1\end{align*}}

\newcommand{\minf}[1]{I(#1)}

\newcommand{\entr}[1]{\mathrm{H}(#1)}

\newcommand{\centr}[2]{\mathrm{H}^{#1}(#2)}

\newcommand{\trainset}{S}

\newcommand{\parspace}{\Theta}
\newcommand{\Dist}{\mathcal D}
\newcommand{\dataspace}{\mathcal Z}
\newcommand\optparen[1]{\ifthenelse{\equal{#1}{}}{}{(#1)}}

\newcommand{\dist}{\ \sim\ }
\newcommand{\unifdist}{\text{Unif}}

\newcommand{\Naturals}{\mathbb{N}}

\newcommand{\Reals}{\mathbb{R}}

\newcommand{\Nats}{\mathbb{N}}

\newcommand{\as}{\textrm{a.s.}}

\newcommand{\dee}{\mathrm{d}}

\DeclareMathOperator*{\newlim}{\mathrm{lim}\vphantom{\mathrm{infsup}}}
\DeclareMathOperator*{\newmin}{\mathrm{min}\vphantom{\mathrm{infsup}}}
\DeclareMathOperator*{\newmax}{\mathrm{max}\vphantom{\mathrm{infsup}}}
\DeclareMathOperator*{\newinf}{\mathrm{inf}\vphantom{\mathrm{infsup}}}
\DeclareMathOperator*{\newsup}{\mathrm{sup}\vphantom{\mathrm{infsup}}}
\renewcommand{\lim}{\newlim}
\renewcommand{\min}{\newmin}
\renewcommand{\max}{\newmax}
\renewcommand{\inf}{\newinf}
\renewcommand{\sup}{\newsup}

\newcommand{\ProbMeasures}[1]{\mathcal{M}_1(#1)}

\renewcommand{\Pr}{\mathbb{P}}
\def\EE{\mathbb{E}}

\newcommand{\defn}[1]{\textit{#1}}

\newcommand{\norm}[1]{\left\Vert #1 \right\Vert}

\newcommand{\iid}{i.i.d.}

\newcommand{\KLname}{\mathrm{KL}}

\newcommand{\KL}[2]{\KLname(#1 \,\|\,#2)}

\newcommand{\e}{\mathrm{e}}

\newcommand{\Alg}{\mathcal{A}}

\newcommand{\unif}[1]{\text{Unif}(#1)}

\newcommand{\bernoulli}{\text{Ber}}

\newcommand{\lcrx}[4][{-1}]{
	\IfEq{#1}{-1}{\left #2 {{{{#3}}}} \right #4}{
   	\IfEq{#1}{0}{#2 {{{{#3}}}} #4}{
	\IfEq{#1}{1}{\bigl #2 {{{{#3}}}} \bigr #4}{
	\IfEq{#1}{2}{\Bigl #2 {{{{#3}}}} \Bigr #4}{
	\IfEq{#1}{3}{\biggl #2 {{{{#3}}}} \biggr #4}{
	\IfEq{#1}{4}{\Biggl #2 {{{{#3}}}} \Biggr #4}{
    \GenericWarning{"4th argument to lcrx must be -1, 0, 1, 2, 3, or 4"}
    }}}}}}}
\newcommand{\inner}[3][{-1}]{\lcrx[#1] < {{#2},{#3}} >}

\newcommand{\sbra}[2][{-1}]{\lcrx[#1] [ {#2} ] }

\newcommand{\rnderiv}[2]{\frac{\text{d} #1}{\text{d} #2}}

\newcommand{\indep}{\mathrel{\perp\mkern-9mu\perp}}

\newcommand{\dminf}[2]{I^{#1} (#2)}

\newcommand{\indic}[1]{ \mathds{1}\left[#1\right]}

\newcommand{\supersample}{\tilde{\pmb{Z}}} %

\newcommand{\range}[1]{ [#1] }

\newcommand{\binaryentr}[1]{\mathrm{H}_{b}(#1)}

 \newcommand{\EGE}{\ensuremath{\mathrm{EGE}_{\Dist}(\Alg_n)}}

\newcommand{\cmi}{\ensuremath{\mathrm{CMI}_{\Dist}(\Alg_n)}\xspace}

\newcommand{\Empriskcvx}[1]{\hat{\mathrm{F}}_{S_n}(#1)}

\newcommand{\Popriskcvx}[1]{\mathrm{F}_{\Dist}(#1)}
\newcommand{\losscvx}[0]{f} %
\newcommand{\proj}[0]{\Pi} %
\newcommand{\proberror}{\mathsf{P_e}} %
\DeclareMathOperator*{\argmin}{arg\,min} %
\newcommand{\clb}{\ensuremath{\mathcal{C}_{L,R}}\xspace} %
\newcommand{\scl}{\ensuremath{\mathcal{C}_{L,\lambda}}\xspace} %

\newcommand{\SCOprob}{(\parspace,\dataspace,\losscvx)}

\usepackage[capitalize,noabbrev]{cleveref}

\title{Information Complexity of Stochastic Convex Optimization:\\ Applications to Generalization and Memorization}
\author{Idan Attias\thanks{Department of Computer Science, Ben-Gurion University and Vector Institute.} 
\and 
Gintare Karolina Dziugaite\thanks{Google DeepMind.} 
\and
Mahdi Haghifam\thanks{Khoury College of Computer Sciences, Northeastern University.} 
\and
Roi Livni\thanks{Department of Electrical Engineering, Tel Aviv University.}
\and
Daniel M. Roy\thanks{Department of Statistical Sciences, University of Toronto and Vector Institute.}
}

\date{}

\usepackage{titlesec}
\titlespacing{\section}{0pt}{\parskip}{0pt}
\titlespacing{\subsection}{0pt}{\parskip}{0pt}
\titlespacing{\subsubsection}{0pt}{\parskip}{0pt}
\usepackage[letterpaper, margin=1in]{geometry}
\newtheorem*{theorem*}{Theorem}
\renewcommand{\epsilon}{\varepsilon}
\usepackage{crossreftools}
\newcommand{\proofsubsection}[1]{\subsection{Proof of \texorpdfstring{\cref{#1}}{\crtcref{#1}}}}

\def\[#1\]{\begin{equation}\begin{aligned}#1\end{aligned}\end{equation}}
\def\*[#1\*]{\begin{equation*}\begin{aligned}#1\end{aligned}\end{equation*}}
\usepackage{xspace}
\newcommand{\ignore}[1]{}

\begin{document}

\renewcommand{\thefootnote}{}
\footnotetext{Authors listed alphabetically. Correspondence: m.haghifam[at]northeastern.edu}
\maketitle

\begin{abstract}
In this work, we investigate the interplay between memorization and learning in the context of \defn{stochastic convex optimization} (SCO). We define memorization via the information a learning algorithm reveals about its training data points. 
We then quantify this information using the framework of conditional mutual information (CMI) proposed by \citet{steinke2020reasoning}. 
Our main result is a precise characterization of the tradeoff between the accuracy of a learning algorithm and its CMI, answering an open question posed by \citet{livni2023information}. We show that, in the  
$L^2$ Lipschitz--bounded setting and under strong convexity, every learner with an excess error $\epsilon$ has CMI bounded below by $\Omega(1/\epsilon^2)$ and  $\Omega(1/\epsilon)$, respectively.
We further demonstrate the essential role of memorization in learning problems in SCO by designing an adversary capable of accurately identifying a significant fraction of the training samples in specific SCO problems. Finally, we enumerate several implications of our results, such as a limitation of generalization bounds based on CMI and the incompressibility of samples in SCO problems.

\end{abstract}

\section{Introduction}

Despite intense study, the relationship between generalization and memorization in machine learning has yet to be fully characterized.
Classically, ideal learning algorithms would primarily extract \emph{relevant information} from their training data, avoiding memorization of irrelevant information. 
This intuition is supported by theoretical work demonstrating the benefits of limited memorization for strong generalization \citep{littlestone1986relating, RussoZou15,RussoZou16,XuRaginsky2017, bassily2018learners, steinke2020reasoning}. 

This intuition, however, is challenged by the success of modern overparameterized deep neural networks (DNNs). These models often achieve high test accuracy despite memorizing a significant number of training data (see, e.g., \citep{Rethinking17,shokri2017membership,carlini2019secret,feldman2020neural,carlini2022quantifying}).
Recent studies suggest that memorization plays a more complex role in generalization than previously thought: %
memorization might even be \emph{necessary} for good generalization \cite{feldman2020does, feldman2020neural, brown2021memorization}.

In this work, we investigate the interplay between generalization and memorization in the context of \defn{stochastic convex optimization} (SCO; \citealt{shalev2009stochastic}). 
A (Euclidean) SCO problem is defined by a triple $\SCOprob$,
where $\parspace \subseteq \Reals^d$ is a convex subset and
$f:\parspace \times \dataspace \to \Reals$ is convex in its first argument for every fixed second argument.
In such an SCO problem, a learner receives a finite sample of data points in the dataspace, $\dataspace$, presumed to be drawn i.i.d.\ from an unknown data distribution, $\Dist$. 
The goal of the learner is to find an approximate minimizer of the population risk
$\Popriskcvx{\theta}\triangleq \EE_{Z\sim \Dist}\left[\losscvx(\theta,Z)\right]$.

In recent years, SCO has been shown to serve as a useful theoretical model for understanding generalization in modern machine learning \citep{feldmanerm,dauber2020can, amir2021never,amir2021sgd,koren2022benign}. The importance of SCO can be traced to a number of factors, including: (1) it is suitable for studying gradient-based optimization algorithms, which are the workhorse behind state-of-the-art machine learning algorithms; 
and
(2) while arbitrary empirical risk minimizers (ERMs) require sample complexity that scales with the problem dimension \citep{feldmanerm, carmon2023sample}, carefully designed algorithms can achieve optimal generalization with sample complexity independent of dimension \cite{bousquet2002stability, shalev2009stochastic}. 
This property aligns with our goal of studying generalization in overparameterized settings such as
DNNs, where first-order methods output models that generalize well, despite the fact that there exist ERMs that perform poorly \citep{Rethinking17}.

\ignore{Consider a (randomized) learning algorithm  $\Alg = (\Alg_n)_{n\ge 1}$ that  selects a parameter $\hat{\theta}= \Alg(S_n)$ based on $n$ \iid\ samples, $
\smash{S_n \sim \Dist^{\otimes n}}
$, where the data distribution $\Dist$ is presumed unknown.
The goal of the algorithm is to find an approximate minimizer 
of the %
\defn{population risk} 
$\Popriskcvx{\theta} \triangleq \EE_{Z \sim \Dist} [\losscvx(\theta,Z)]$. Following the classical framework of \citet{Vapnik2015}, we focus on the distribution-free (worst-case) performance of a learning algorithm: we say a learning algorithm $\Alg$ is an \emph{$\epsilon$-learner of $\SCOprob$} if, for \emph{every} data distribution $\Dist$, $\Popriskcvx{\Alg_n(S_n)}- \min_{\theta \in \Theta} \Popriskcvx{\theta}\leq \epsilon$ with high probability. We also refer to the minimum number of samples $n$ that $\Alg$ requires to achieve $\epsilon$ excess error as the \defn{sample complexity}.}

To shed light on the role of memorization in SCO, we analyze the information-theoretic properties of $\epsilon$-learners for SCO problems: we say a learning algorithm $\Alg=\{\Alg_n\}_{n\geq 1}$ is an \emph{$\epsilon$-learner of $\SCOprob$} if for sufficiently large $n$, for \emph{every} data distribution $\Dist$, $\Popriskcvx{\Alg_n(S_n)}- \min_{\theta \in \Theta} \Popriskcvx{\theta}\leq \epsilon$ with high probability over the draws of the training set $\trainset_n \sim \Dist^{\otimes n}$ and the randomness of $\Alg$. The current paper revolves around the following fundamental question:
\begin{center}
\emph{How much information must an $\epsilon$-learner reveal about their training data?}
\end{center}

To address this question, we study the mutual information between (various summaries of) the learner's outputs and the training set, possibly conditional on other quantities. Early work along these lines, due to \citet{XuRaginsky2017} (see also foundational work by \citep{RussoZou15,RussoZou16,bu2020tightening} and \citep[][App.~C]{negrea2019information})
provided 
information-theoretic generalization bounds based on
the mutual information between the full training sample and the output hypothesis (the so-called \emph{input--output mutual information}, or IOMI). 
Recently, \citet{livni2023information} demonstrated a fundamental lower bound on the IOMI $\epsilon$-learners in the context of SCO: for every algorithm, its IOMI scales with the dimension $d$.
Regarding whether studying IOMI sheds light on memorization, there is an important caveat with \citep{livni2023information}: 
bits of information between the sample and the model do not distinguish between the number of bits per sample and the number of memorized samples. 
In particular, the work of \citet{livni2023information} does not rule out the sufficiency of memorizing a single example, which overall has $O(d)$ entropy.

To remedy this, our work introduces a refined perspective on capturing memorization, focusing on 
\defn{conditional mutual information} (CMI) as a notion of information complexity developoed by \citet{steinke2020reasoning}. 
CMI quantifies the amount of information that the learner's output reveals about its training sample, conditioned on a ``super sample'', from which the training sample is taken. (Formal definitions are provided in \cref{sec:preliminaries}.) 
Contrasted with the bound in \citep{XuRaginsky2017},
in this setup, the memorization of a single example provides at most one bit of information. In other words, the scale of the CMI  is more instructive on the \emph{number} of memorized samples. Can we use CMI to fully characterize the interplay between memorization and learning in SCO?

\subsection{Contributions}

Our main result is a precise characterization of the tradeoff between the accuracy of a learning algorithm and its CMI:

\textbf{Key result: CMI--Accuracy Tradeoff for $\epsilon$-learners.}

We show that in the general SCO setup as well as under further structural assumption of strong convexity, there exists a tradeoff between the accuracy of an $\epsilon$-learner and its CMI: Surprisingly, to achieve small excess error, a learner \emph{must} carry a large amount of CMI, scaling with the optimal sample size. This result completely answers an open question by \citet{livni2023information}. More precisely, we study the CMI of learners for two important classes of SCO problems:
\begin{itemize}[leftmargin=1em]
\item \emph{Lipschitz bounded SCO:} We construct an SCO problem such that, for every $\epsilon$-learner, there exists a distribution such that the CMI of the learner is $\Omega(1/\epsilon^2)$, despite the already-established optimal sample complexity $O(1/\epsilon^2)$. We complement this result with a matching upper bound. We also show that this result holds for both proper as well as improper (unconstrained) learning algorithms. 
\item \emph{Strong Convexity:} Under further structural assumption of strong convexity, we establish an $\Omega(1/\epsilon)$ lower bound on CMI of every $\epsilon$-learner which we show is also tight.
\end{itemize}
Our proof techniques are inspired by the differential privacy literature and build on
so-called \emph{fingerprinting lemmas} \citep{bun2014fingerprinting,steinke2016upper,kamath2019privately}. Our key results and proof ideas have various interesting implications:

\paragraph{Limitation of the CMI Generalization Bound for SCOs.} Our lower bounds highlight that CMI-based generalization bounds for SCO do not fully explain the optimal excess error. 
For algorithms with optimal sample complexity, the established CMI lower bound 
implies that standard CMI generalization guarantees are vacuous.

In more detail, \citet{steinke2020reasoning} show that the generalization error of any learner can be bounded by
\[\nonumber \textrm{generalization gap} \le \sqrt{\frac{\cmi}{n}}.\]

(See \cref{sec:preliminaries} for formal definitions.) Plugging our lower bound on CMI into the above equation we obtain an upper bound on the generalization gap of $O\left(\sqrt{\frac{1}{\epsilon^2\cdot n}}\right)$ which is strictly larger than the true $O(\epsilon)$ error. In particular, for the optimal choice of $n$, we obtain a vacuous generalization bound of order $\Omega(1)$, even though the algorithm perfectly learns. 
Similarly, under the assumption of strong convexity, one can learn with sample complexity of $O(1/\epsilon)$. Thus, again we obtain that the CMI bound may be an order of $\Omega(1)$, even though the learner is able to learn.

\paragraph{Necessity of Memorization.}
Inspired by the CMI and membership inference \citep{carlini2022membership}, 
we have developed a framework to quantify memorization in SCO: 
informally, a point is considered memorized if an adversary can guess correctly if this point appeared in the training set with high confidence.
Building on our construction for CMI, 
we design an adversary capable of correctly identifying a significant fraction of the training samples in certain SCO problems, implying that memorization is a necessary component in this context. A similar point appeared in \citep{feldman2019high,feldman2020does,brown2021memorization}.

To be more precise, we consider a contestant and an adversary. The contestant gets to train a model on a training set not revealed to the adversary. The contestant then shows the adversary a sample either from the training set or a freshly drawn sample (not seen during training time). A point is considered \emph{memorized} if the adversary correctly identifies whether the shown sample appeared during training time (while refraining from accusing freshly drawn samples). 

We show that our approach for lower bounding CMI lets us design an adversary with the following guarantee: there exists a natural SCO problem such that for every $\epsilon$-learner, there exists a distribution such that the adversary can distinguish $\Omega(1/\epsilon^2)$ of the training samples with high confidence. We also establish a similar result under an additional assumption of strong convexity, showing that there exists an adversary that can distinguish $\Omega(1/\epsilon)$ of the training samples. Notice that in both cases,
the size of the sample to be memorized scales linearly with the sample complexity. In other words, any sample-efficient learner needs to memorize a constant fraction of its training set.

\paragraph{Incompressibility of Samples in SCOs.}
Our results rule out the existence of constant-sized (dimension-independent) sample compression schemes for SCO. 
Many learning algorithms, like Support Vector Machine (SVM), generate their output using only a small subset of training examples— for SVM such a subset is known as support vectors. \emph{Sample compression schemes}, introduced by \citet{littlestone1986relating}, provide a precise characterization of this algorithmic property. 
Since the optimal sample complexity in SCO is dimension-independent, a natural question to ask is whether we can construct a sample compression scheme of \emph{constant} size for SCOs. (Here \emph{constant compression size} refers to a dimension-independent quantity.) Using the results connecting the CMI and sample compression schemes in \citep{steinke2020reasoning}, we show that such a construction is impossible. This finding is in stark contrast with binary classification \citep{moran2016sample, david2016supervised} and regression \citep{hanneke2019sample, attias2024agnostic} in the PAC setting, where, in this context, constant size compression depends only on the VC dimension and the fat-shattering dimension, respectively.
This is a long-standing open question of whether the optimal sample complexity can be obtained in the PAC setting based on sample compression schemes; while there are known sample compression schemes of constant size, they have exponential dependence in the relevant combinatorial dimension.
Our result rules out the possibility of obtaining optimal sample complexity in SCO based on sample compression schemes. 

\paragraph{Individual-Sample variant of CMI.}
We demonstrate that our techniques extend to lower-bounding the individual sample variants of CMI as proposed in \citep{haghifam2020sharpened,rodriguez2020random,zhou2020individually}. These individual sample variants of CMI have been shown to provide tighter generalization measures compared to standard CMI \citep{haghifam2020sharpened,rodriguez2020random,zhou2020individually}. However, our results show that in the context of SCO, no such improvement is possible, and the same lower bound holds.

\subsection{Organization}
The rest of this paper is structured as follows. In \cref{sec:related-work} we discuss the related work. After providing the necessary preliminaries in  \cref{sec:preliminaries}, we present an overview of the main results in \cref{sec:main-results}.  Then, in \cref{sec:implications}, we discuss several implications of our main results. Finally, in \cref{sec:cvx} and \cref{sec:scvx}, we present the key steps of the proofs of the main results.

\section{Related Work}
\label{sec:related-work}
\paragraph{Information-Theoretic Measures of Generalization.}
In recent years, there has been a flurry of interest in the use of information-theoretic quantities for characterizing the expected generalization error of
learning algorithms. For an excellent overview of recent advances see \citep{alquier2021user,hellstrom2023generalization}.  Here, we discuss the work on worst-case information-theoretic measures of learning algorithms. The initial focus of this line of work \citep{RussoZou15,RussoZou16,XuRaginsky2017} was based on \emph{input--output mutual information (IOMI)} of an algorithm. Unfortunately, IOMI does not yield a useful notion of information complexity for learning in many key settings. 
For instance, prior work highlights severe limitations of the IOMI framework in the settings of binary classification \citep{bassily2018learners,nachum2018direct,roishay} and SCO \citep{livni2023information}. In PAC learning,
there exist hypothesis classes, such as thresholds, that are learnable even though the IOMI is unbounded. In SCO, while the optimal sample complexity does not scale with the dimension, the IOMI can be unbounded (dimension-dependent). 

The notion of CMI \cite{steinke2020reasoning,grunwald2021pac,harutyunyan2021information,haghifam2021towards,haghifam2022isit,hellstrom2022evaluated} remedies some of the above issues, at least in the classification setting.
 While the CMI addresses some of the limitations of IOMI, \citet{haghifam2023limitations} show that it cannot explain the minimaxity of gradient descent in SCO. 
Our work significantly extends their result: We show that the same limitations hold for \emph{every} $\epsilon$-learner algorithm with a dimension-independent sample complexity. 
Notice that gradient descent with a proper learning rate \cite{bassily2020stability, amir2021never} is one of the $\epsilon$-learner algorithms that can have dimension-independent sample complexity. 
See \cref{remark:alt-paper} for a detailed discussion. A recent work of \citet{wang2023sample} proposes a new measure similar to CMI referred to as hypotheses-conditioned CMI and shows that it is related to the uniform stability \citep{bousquet2002stability}.
However, hypotheses-conditioned CMI is not an appropriate measure for studying memorization in SCO since its conditioning term is different. The structure used to define CMI inspired    \citet{sachs2023generalization} to introduce the \emph{algorithmic-dependent Rademacher
Complexity}. We leave the problem of studying the separation between CMI and algorithmic-dependent Rademacher
Complexity in the context of SCO as a future direction.

\paragraph{Memorization.}
Theoretical aspects of the necessity of memorization in learning have been recently studied
\citep{feldman2020neural,feldman2020does,brown2021memorization,brown2022strong}. 
The measure of memorization in our work differs from prior work. Additionally, none of the previous studies examined the question of memorization in the context of SCO.
Most similar to our work is \citep{brown2021memorization} where the authors study memorization using IOMI. 
Memorization has been demonstrated to happen also empirically in state-of-the-art algorithms \cite{carlini2019secret, carlini2021extracting, haim2022reconstructing,carlini2022membership}. In contrast with empirical studies, the aim of a theoretical investigation is to study its role, and whether it is \emph{necessary} or a byproduct of current practices.

\paragraph{Fingerprinting Codes and Privacy Attacks.}
The key idea behind our lower bound proof builds on privacy attacks developed in differential privacy known as \emph{fingerprinting codes} \citep{boneh1995collusion,tardos2008optimal,bun2014fingerprinting,steinke2016upper,kamath2019privately}. \citet{tracing} consider the problem of designing privacy attacks on the mean estimators that expose a fraction of the training data. 
They propose an adversary, demonstrating that every algorithm that precisely estimates mean in $\ell_\infty$ leaks the membership of the samples in the training set. The $\ell_\infty$ hypercube cannot be learned in a dimension-independent sample size. Therefore, to obtain the separation we desire, we can only assume a weaker $\ell_2$ approximation, which leads to further challenges, especially in the unconstrained non-strongly convex case, which is the hardest.

\section{Preliminaries}\label{sec:preliminaries}

\paragraph{Notations}
Let $d\in \Naturals$. For $x\in \Reals^d$, $\norm{x}$ denotes $\ell_2$ norm of $x$, and $\inner{\cdot}{\cdot}$ denotes the standard inner product in $\Reals^d$. For $k\in \range{d}$, we denote the $k$-th coordinate of a $d$-dimensional vector $x$ by the superscript $x^{(k)}$. For a matrix $A\in \Reals^{n\times m}$, $\norm{A}_2$ is the operator norm of $A$. $\mathcal{B}_d(1)$ denote the ball of radius one in $\Reals^d$. For
a (measurable) space $\mathcal{R}$, $\ProbMeasures{\mathcal{R}}$ denotes the set of all probability measures on $\mathcal{R}$. Finally, let $\indic{\cdot}$ denote the indicator function: $\indic{p}=1$ if predicate $p$ is true, and $\indic{p}=0$ otherwise. 

\subsection{Background on Information Theory}
Let $P,Q$ be probability measures on a measurable space.
When $Q$ is absolutely continuous with respect to $P$, denoted $Q \ll P$, we
write $\rnderiv{Q}{P}$ for (an arbitrary version of) the 
Radon--Nikodym derivative (or density) of $Q$ with respect to $P$. 
The \defn{KL divergence} (or \defn{relative entropy}) of \defn{ $Q$ with respect to $P$},
denoted $\KL{Q}{P}$, equals $\int \log \rnderiv{Q}{P} \dee Q$ when $Q \ll P$, and is
infinity otherwise. The \defn{mutual information between $X$ and $Y$} 
is $$
\minf{X;Y} = \KL{\Pr[(X,Y)]} { \Pr[X] \otimes \Pr[Y]},
$$
where $\otimes$ forms the product measure. The \defn{disintegrated mutual information between $X$ and $Y$ given $Z$} is 
$$
 \dminf{Z}{X;Y} =  \KL{ \Pr\left((X,Y) \big| Z\right)}{ \Pr\left(X \big| Z\right) \otimes  \Pr\left(Y \big| Z\right) },$$
 where $\Pr\left(Y \big| Z\right)$ is the conditional distribution of $Y$ given $Z$. Then, the conditional mutual information is 
 $$
 \minf{X;Y\vert Z}=\EE{[\dminf{Z}{X;Y}]}.
 $$

 If $X$ concentrates on a countable set $V$ with counting measure $\nu$, 
the \defn{(Shannon) entropy of $X$} is 
$
\entr{X} = - \sum_{x\in V} \Pr(X=x)\,\log \Pr (X=x )
$. The \defn{disintegrated entropy of $X$ given $Y$} is defined by 
$
\centr{Y}{X} = - \sum_{x\in V}\Pr\left(X=x\big| Y\right) \log \Pr\left(X=x\big| Y\right) ,
$
while the \defn{conditional entropy of $X$ given $Y$} is $\entr{X\vert Y} =  \EE[\centr{Y}{X}]$. Also, for $p \in [0,1]$, the binary entropy function is given by $\binaryentr{p}\triangleq -p \log(p) - (1-p)\log(1-p)$ with the assumption that $0 \log(0)=0$.
\subsection{Stochastic Convex Optimization (SCO)}
A \defn{stochastic convex optimization} (SCO) problem is a triple $\SCOprob$, where $\parspace \subseteq \Reals^d$ is a convex set and $\losscvx(\cdot,z) : \parspace \to \Reals$ is a convex function for every $z \in \dataspace$. We refer to $\parspace$ as the parameter space,
to its elements as parameters, to elements of $\dataspace$ as data, and to $f$ as the \defn{loss function}. Informally, given an SCO problem $\SCOprob$, the goal is to find an approximate minimizer 
of the %
\defn{population risk} 
$\Popriskcvx{\theta} \triangleq \EE_{Z \sim \Dist} [\losscvx(\theta,Z)],$
given an i.i.d.\ sample $\trainset_n = \{Z_1, \ldots, Z_n\}$ drawn from an unknown distribution $\Dist$ on $\dataspace$, denoted by $\Dist \in \ProbMeasures{\dataspace}$.  The \defn{empirical risk} of $\theta \in \parspace$ on a sample $\trainset_n\in \dataspace^{n}$ is 
$\Empriskcvx{\theta} := \frac{1}{n} \sum_{i \in [n]} f(\theta,Z_i)$,
where $\range{n}$ denotes the set $\{1,\dots,n\}$.
A \defn{learning algorithm} is a sequence 
$\Alg = (\Alg_n)_{n\ge 1}$ such that,
for every positive integer $n$, $\Alg_n$ maps $\trainset_n$ to a (potentially random) element $\hat{\theta}=\Alg_n(\trainset_n)$ in $\Reals^d$. 
The \defn{expected generalization error} of $\Alg_n$ under $\Dist$ is
$
\nonumber
 \EGE = \EE \sbra[0]{\Popriskcvx{\Alg(\trainset_n)}- \Empriskcvx{\Alg(\trainset_n)} }
$. Also, the expected excess error $\Alg_n$ under $\Dist$ is $ \EE \sbra[0]{\Popriskcvx{\Alg(\trainset_n)}}- \min_{\theta \in \parspace } \Popriskcvx{\theta}$. A learning algorithm is called \emph{proper} if its output, for all possible training sets, satisfies $\Alg_n(\trainset_n)\in \parspace$. Otherwise, it is called \emph{improper}.

\begin{definition}($\epsilon$-learner for SCO)
\label{def:eps-learner}
Fix an SCO problem $\SCOprob$ and $\epsilon>0$. We say $\Alg = \{\Alg_{n}\}_{n\geq 1}$ \emph{$\epsilon$-learns} $\SCOprob$ with sample complexity of $N:\Reals \times \Reals \to \Naturals$ if the following holds: for every $\delta \in (0,1]$, given number of samples $n \geq N(\epsilon,\delta)$, we have that for every $\Dist \in \ProbMeasures{\dataspace}$, with probability at least $1-\delta$ over $\trainset_n \sim \Dist^{\otimes n}$ and internal randomness of $\Alg$,
$$
\displaystyle \Popriskcvx{\Alg_n(\trainset_n)}-\min_{\theta \in \parspace } \Popriskcvx{\theta} \leq \epsilon.
$$
We also refer to $N(\cdot,\cdot)$ as  \defn{sample complexity} of $\Alg$.
\end{definition}

We consider two important subclasses of SCO problems that impose different conditions over the loss function and the parameter space \citep{shalev2014understanding,shalev2009stochastic}.
\begin{enumerate}
    \item \defn{Convex-Lipschitz-Bounded (CLB):} SCO with convex and $L$-Lipschitz loss function defined over a bounded domain with diameter $R$, namely, for any $\theta \in \Theta$ we have $\norm{\theta}\leq R$. We say a loss function is $L$-Lipschitz if and only if $\forall z \in \dataspace$, $\forall \theta_1,\theta_2 \in \parspace: |f(\theta_1,z)-f(\theta_2,z)|\leq L \norm{\theta_2-\theta_1}$. We refer to this subclass as \clb. 
    \item \defn{SCO with $L$-Lipschitz and $\lambda$-strongly convex loss (CSL):} We say a loss function is $\lambda$-strongly convex for all $\theta_1,\theta_2 \in \parspace$ and $z\in \dataspace$ we have $f(\theta_2,z) \geq f(\theta_1,z) + \inner{\partial f(\theta_1,z)}{\theta_2-\theta_1} + \frac{\lambda}{2}\norm{\theta_2-\theta_1}^2$ where $\partial f(\theta_1,z)$ is the subgradient of $\losscvx(\cdot,z)$ at $w$. The definition of Lipschitzness is the same as in the CLB subclass.  We refer to this subclass as \scl.
\end{enumerate}

\subsection{Measure of Information Complexity}
Next, we formally introduce the framework proposed by \citet{steinke2020reasoning} which aims to quantify the information complexity of a learning algorithm.
\begin{definition}\label{def:cmi}
Let $\Dist$ be a data distribution, and $\Alg = (\Alg_n)_{n\geq 1}$ a learning algorithm. For every $n \in\Naturals$, let $\supersample = (Z_{j,i})_{j \in \{0,1\},i\in \range{n}}$ be an array of i.i.d. samples drawn from $\Dist$, and $U=(U_1,\dots,U_n) \sim \bernoulli\left(\frac{1}{2}\right)^{\otimes n}$, where $U$ and $\supersample$ are independent. Define a training set $\trainset_n = (Z_{U_i,i})_{i  \in \range{n}}$. The conditional mutual information (CMI) of $\Alg_n$ with respect to $\Dist$ is 
$$
\displaystyle \cmi \triangleq \minf{\Alg_n(\trainset_n);U\vert \supersample}.
$$
\end{definition}

\section{Main Results} \label{sec:main-results}
In this section, we formally state our main results. First in \cref{sec:main-results-cmi}, we give an overview of the CMI-accuracy tradeoff for $\epsilon$-learners. Then, in \cref{sec:memorization_game}, we precisely define the memorization game and present our results on the necessity of memorization.
\subsection{CMI-Accuracy Tradeoff} \label{sec:main-results-cmi}
We begin with a lower bound on the CMI for the CLB subclass.
\begin{theorem}[\textbf{CMI-accuracy tradeoff}]\label{thm:main-lower-convex}
Let $\epsilon_0 \in (0,1)$ be a universal constant.
There exists a loss function $f(\cdot,z)$ that is $1$-Lipschitz, for every $z$ such that:
For every $\epsilon \leq \epsilon_0$ and for every algorithm $\Alg=\{\Alg_n\}_{n \in \Naturals}$ that $\epsilon$-learns with sample complexity $N(\cdot,\cdot)$ the following holds: for every $\delta \leq \epsilon$, $n\geq N(\epsilon,\delta)$, and $d\geq \Omega(n^2\log(n))$, there exists a data distribution $\Dist \in \ProbMeasures{\dataspace}$ such that
\[
\nonumber
\cmi = \Omega\left(\frac{1}{\epsilon^2}\right).
\]
\end{theorem}
In particular, we obtain that for every algorithm, in sufficiently large dimension, there exists a problem instance where the CMI-generalization bound in \cite{steinke2020reasoning} becomes vacuous for every algorithm with optimal sample complexity $n=O(1/\epsilon^2)$. Note that the Theorem above holds for $\epsilon$-learner with arbitrary sample size.

Notice that the bound above is tight; namely, there exists an $\epsilon$-learner with CMI at most $O(1/\epsilon^2)$. Consider a base algorithm with the sample complexity $N(\epsilon,\delta) =\Omega\left(\log \left(1/\delta\right) /\epsilon^2\right)$ (e.g. regularized ERM \cite{bousquet2002stability} or stabilized Gradient Descent \cite{bassily2020stability}). Then, given $n\geq \Omega\left(\log \left(1/\delta\right) /\epsilon^2\right)$, we may consider an algorithm that subsamples $O(\log \left(1/\delta\right)/\epsilon^2)$ examples and feed it into the base algorithm. By the definition of the CMI, it is bounded by the size of the subsample used for learning. This argument shows that there exists an algorithm with $\cmi=O(1/\epsilon^2)$. The formal statement of the described upper bound appears in \cref{thm:lip-bounded-upperbound}. 

Under further structural assumptions, though, the sample complexity in SCO can be improved. It is a question then if CMI bounds can also be further tightened under structural assumptions such as, for example, strong convexity. Our next result shows that this is indeed the case:

\begin{theorem}[\textbf{CMI-accuracy tradeoff, strongly convex case}]
\label{thm:main-lower-stronglyconvex}
Let $\epsilon_0$ and $\delta_0$ be universal constants.
There exists a function $f(\cdot,z)$ that is $1$-Lipschitz, and $1$-strongly convex, for every $z$ such that:
For every $\epsilon<\epsilon_0$ and $\delta<\delta_0$ and for every $\epsilon$-learner ($\Alg=\{\Alg_n\}_{n \in \Naturals}$), with sample complexity $N(\cdot,\cdot)$ the following holds: for every $n\geq N(\epsilon,\delta)$, $\delta<O(1/n^2)$, and $d\geq \Omega(n^2\log(n))$, there exists a data distribution $\Dist \in \ProbMeasures{\dataspace}$ such that
\[
\nonumber
\cmi \geq \Omega\left(\frac{1}{\epsilon}\right).
\]
\end{theorem}
As in the general case, the above bound is tight. As discussed in \citep{shalev2009stochastic}, any ERM is stable, hence generalizes over a strongly convex objective with sample complexity of $N(\epsilon,\delta)= O(\log (1/\delta)/\epsilon)$. Therefore, as before, we obtain that the above bound is tight for this setup. The formal statement of the upper bound appears in \cref{thm:strong-cvx-upperbound}.

We finish this section by introducing a memorization game that helps us formalize in what sense a learner must memorize the data in SCO.

\subsection{Memorization Game}\label{sec:memorization_game}

Intuitively, we can think of CMI as measuring the number of examples we can identify from the training set by observing the model. However, formally there is a gap between this interpretation and the definition of CMI. For example, one could think of a learner that \emph{spreads the information} by using many samples, where we have that $\cmi \ge \Omega(1/\epsilon^2)$, but for each specified example, the information over $U_i$ is small (see \cref{def:cmi}.). In other words, there is a formal gap between large CMI and intuitive notions of memorization. In this subsection, we aim to close this gap by showing that, in fact, this is not the case, and the information the learner carries on $U$ can be used to actually identify examples from the training set. The proofs appear in \cref{appx:memorizatin}.

\begin{definition}[\textbf{Recall Game for $i$-th example}]
\label{def:mem-game}
Let $\Alg=\{\Alg_n\}_{n\geq 1}$ be a learning algorithm, $\trainset_n = (Z_1,\dots,Z_n)\sim \Dist^{\otimes n}$ be a training set, and  $\hat{\theta}=\Alg_n(\trainset_n)$. Let $\mathcal{Q}:\Reals^d \times \dataspace  \times \ProbMeasures{\dataspace}   \to \{0,1\}$ be an adversary. Consider the following game. For $i\in [n]$, we sample a fresh data point $\tilde{Z}_i \sim \Dist$, independent of $\hat \theta$ and $Z_i$. Let $Z_{1,i}=Z_i$ and  $Z_{0,i}=\tilde{Z}_i$. Then, we flip a fair coin $b_i \sim \unif{\{0,1\}}$. Finally, the adversary outputs $\hat b_i \triangleq \mathcal{Q}\left(\hat{\theta},Z_{b_i,i},\Dist\right)$.
\end{definition}
The next definition formalizes the measures used for evaluating an adversary.
\begin{definition}[\textbf{Soundness and recall}]
    Consider the setup described in \cref{def:mem-game}. Assume that the adversary plays the game for each of the data points in the training set, i.e., $n$ rounds. Then, 
    \begin{enumerate}
        \item We say the adversary is $\xi$-sound if $\Pr\left(\exists i \in \range{n}\!:\! \mathcal{Q}\left(\hat{\theta},Z_{0,i},\Dist\right)=1 \right)\leq \xi$ where $\xi\in[0,1]$ is a constant.
       
        \item We say the adversary certifies the recall of $m$ samples if  $\Pr\left(\sum_{i=1}^{n}  \mathcal{Q}\left(\hat{\theta},Z_{1,i},\Dist\right) \geq m  \right) \geq \Omega(1)$ where $\Omega(1)$ denote a universal constant (up to log factors in the parameters of problems). 
    \end{enumerate}
\end{definition}

Intuitively, the soundness condition implies that if the adversary identifies a sample as part of the training set, its prediction needs to be accurate. Then, the recall condition makes sure the adversary can identify many training points, which is quantified by $m$. There is a tradeoff between the constant probability of certification and the size of samples that can be recalled. 
Next, we present the main results of memorization:

\begin{theorem}[\textbf{Memorization/membership inference attack}]
\label{thm:membership-cvx}
Let $\epsilon_0 \in (0,1)$ be a universal constant. Fix $\xi \in (0,1]$. There exists an SCO problem with $1$-convex Lipschitz loss defined over the ball of radius one in $\Reals^d$, and there exists an efficient adversary such that the following holds. For every $\epsilon<\epsilon_0$, $\delta<\epsilon$, and for every $\epsilon$-learner ($\Alg$), with sample complexity $N(\epsilon,\delta)=\Theta\left(\log(1/\delta)/\epsilon^2\right)$ the following holds: for $n=N(\epsilon,\delta)$ and $d\geq \Omega(n^2\log(n/\xi))$, there exists a data distribution $\Dist \in \ProbMeasures{\dataspace}$ such that
the adversary is $\xi$-sound and certifies a recall of $\Omega(1/\epsilon^2)$ samples.
\end{theorem}

\begin{theorem}[\textbf{Memorization/membership inference attack, strongly convex case}]
 \label{thm:membership-scvx}
Let $\epsilon_0$ and $\delta_0$ be universal constants. Fix $\xi \in (0,1]$. There exists an SCO problem with $O(1)$ strongly convex and $O(1)$ Lipschitz loss, and there exists an efficient adversary such that the following is true. For every $\epsilon<\epsilon_0$, $\delta<\delta_0$, and for every $\epsilon$-learner ($\Alg$), with sample complexity $N$ the following holds: for every $n \geq N(\epsilon,\delta)$, $\delta<O(1/n^2)$, and $d\geq \Omega(n^2\log(n/\xi))$, there exists a data distribution $\Dist \in \ProbMeasures{\dataspace}$ such that
the adversary is $\xi$-sound and certifies a recall of $\Omega(1/\epsilon)$ samples.
\end{theorem}

\section{Implications}\label{sec:implications}
\subsection{Limitation of CMI-Based Generalization Bounds for SCO}
CMI is proposed by \citet{steinke2020reasoning} as an information-theoretic measure for studying the generalization properties
of learning algorithms. An important question regarding the CMI framework is for which learning problems and learning algorithms is the CMI framework \emph{expressive} enough to accurately estimate the optimal worst-case generalization error? This question has been studied extensively for the setting of binary classification and 0--1 valued loss. In \citep{steinke2020reasoning,grunwald2021pac,haghifam2021towards,harutyunyan2021information,hellstrom2022evaluated}, it has been shown that  CMI framework can be used to establish near-optimal worst-case excess error bounds in the realizable setting. Despite these successful applications, much less is known about the optimality or limitations of the CMI framework beyond the setting of binary classification and 0--1 valued loss. In this section, our main result shows that for every learning algorithm for SCO with an optimal sample complexity, the generalization bound using the CMI framework is vacuous.  First, we start by quoting a result from \citep{haghifam2023limitations} which extends the generalization bounds based on CMI to SCO problems.

\begin{theorem}[\citep{haghifam2023limitations}]
\label{thm:gen-bound-clb-cmi}
Let $n \in \Naturals$, $\Dist \in \ProbMeasures{\dataspace}$ be a data distribution, and $\trainset \dist \Dist^{\otimes n}$. %
Consider an SCO problem $(\losscvx, \parspace, \dataspace) \in \clb$.
Then, for every learning algorithm $\Alg_n$ such that $\Alg_n(\trainset_n)\in \parspace$ \as, %
$\displaystyle \EGE \leq LR \sqrt{{8\cmi}/{n}}$.
\end{theorem}

Consider an SCO problem $\SCOprob \in \clb$. To control the excess population error for an algorithm, a common strategy is bounding it using the generalization and optimization errors:
\begin{align*}
&\EE\left[\Popriskcvx{\Alg_n(\trainset_n)}\right] -\min_{\theta \in \parspace} \Popriskcvx{\theta} 
\leq 
 \EGE + \EE \big[ \Empriskcvx{\Alg_n(\trainset_n)} - \min_{\theta \in \parspace}\Empriskcvx{\theta} \big].
\end{align*}
For the proof, see \citep{haghifam2023limitations,bassily2020stability}. Since we are interested in controlling the $\EGE$ using CMI, we can use \cref{thm:gen-bound-clb-cmi} to further upper-bound the excess error as 
\[
\label{eq:excess-error-decompose-cmi}
&\EE\left[\Popriskcvx{\Alg_n(\trainset_n)}\right] -\min_{\theta \in \parspace} \Popriskcvx{\theta} \leq  
 LR \sqrt{\frac{8\cmi}{n}} + \EE \big[ \Empriskcvx{\Alg_n(\trainset_n)} - \min_{\theta \in \parspace}\Empriskcvx{\theta} \big].
\]
It has been known for every learning algorithm that $\epsilon$-learn the subclass \clb of SCOs, the optimal sample complexity is $\Theta\left(\left(\frac{LR}{\epsilon}\right)^2\right)$\citep{shalev2009stochastic}. A natural question to ask is: \emph{Can the excess error decomposition using CMI accurately capture the worst-case excess error of optimal algorithms for SCOs?} Our next result provides a negative answer to this question.

\begin{theorem}[\textbf{Non-optimality of CMI generalization bound in SCO}]
\label{thm:gen-error-limitation}
For every $L\in \Reals$ and $R \in \Reals$, there exists an SCO problem $\SCOprob \in \clb$ such that the following holds: for every learning algorithm $\Alg=\{\Alg_n\}_{n\in \Naturals}$ with sample complexity $N:\Reals \to \Naturals$ such that for every $\epsilon>0$, $N(\epsilon,\delta)=\tilde{\Theta}\left(\left(\frac{LR}{\epsilon}\right)^2\right)$, there exists a data distribution such that $  LR \sqrt{{8\cmi}/{n}} = \Theta\left(LR\right)$, while the excess error is at most $\epsilon$.
\end{theorem}

\begin{remark}
\label{remark:alt-paper}
In \citep{haghifam2023limitations}, the authors show that for a \emph{particular} algorithm of Gradient Descent (GD), there exists a distribution such that, the upper bound based on CMI is vacuous. 
With the correct choice of learning rate GD can, with an optimal sample complexity, learn the subclass CLB of SCOs. Notice that our result in \cref{thm:gen-error-limitation} significantly extends the limitations proved in \citep{haghifam2023limitations}, by showing that for \emph{every} learning algorithm with an optimal sample complexity, the generalization bound based on CMI is \emph{vacuous}. 
\end{remark}

\subsection{Non-Existence of Sample Compression Schemes in SCO}
Many learning algorithms share the property that their output is constructed using a small subset of the training set. For example, in Support Vector Machine, only the set of support vectors is needed to construct the separating hyperplane in the realizable setting. \emph{Sample compression schemes}, introduced by \citet{littlestone1986relating, floyd1995sample}, provide a formal definition for this algorithmic property accompanied by provable generalization bounds.
These bounds proved to be useful in numerous learning settings, particularly when the uniform convergence property does not hold or provides suboptimal rates,  such as binary classification \cite{graepel2005pac,moran2016sample,bousquet2020proper}, multiclass classification \cite{daniely2015multiclass,daniely2014optimal,david2016supervised,brukhim2022characterization}, regression \cite{hanneke2019sample,attias2024agnostic,attias2023optimal}, active learning \cite{wiener2015compression}, density estimation \cite{ashtiani2020near}, adversarially robust learning \cite{montasser2019vc,montasser2020reducing,montasser2021adversarially,montasser2022adversarially,attias2022characterization,attias2023adversarially}, learning with partial concepts \cite{alon2022theory}, and showing Bayes-consistency for nearest-neighbor methods \cite{gottlieb2014near,kontorovich2017nearest}.
As a matter of fact, compressibility and learnability are known to be equivalent for general learning problems \cite{david2016supervised}.
A remarkable result by \cite{moran2016sample} showed that VC classes enjoy a sample compression that is independent of the sample size.

We define the most general version of a sample compression scheme. Formally, we say a learning algorithm $\Alg_n$  is an $\alpha$-approximate  \defn{sample compression scheme} of size $k \in \Nats$ if there exists a pair $(\kappa,\rho)$ of maps
such that, for all 
sequences $S_n = (Z_i)_{i=1}^{n}$ of size $n \ge k$,
the map $\kappa$ compresses the sample into a length-$k$ subsequence $\kappa(S_n) \subseteq S_n$
which the map $\rho$ uses to reconstruct  the output of the algorithm, i.e., 
 $\Alg_n(S_n)=\rho(\kappa(S_n))$, with near-optimal error on $S_n$ with respect to parameter space $\Theta \subseteq \Reals^d$ and loss function $f$:
\begin{align*}
\frac{1}{n}\sum_{i\in[n]}f(\rho(\kappa(S_n)),Z_i)
\leq 
\min_{\theta\in\Theta}\frac{1}{n}\sum_{i\in[n]}f(\theta,Z_i)+\alpha.
\end{align*}

 Steinke and Zakynthinou prove that for $n\geq k$, if $\Alg_n$ is a sample compression scheme $(\kappa,\rho)$ of size $k$, then for every $\Dist$, $\cmi\leq k\log(2n)$ where $\Alg_n(\cdot)=\rho(\kappa(\cdot))$.

A natural question to ask is: \emph{Can we learn CLB or CSL subclasses of SCOs using sample compression schemes?} In particular, we are interested in sample compression schemes in which $k$ is independent of the dimension and $n$ so that the algorithm has a dimension-independent sample complexity. Using the results presented in the previous sections, we provide a negative answer. 
\begin{corollary}[\textbf{Unbounded sample compression scheme in SCO}]
    Let $\epsilon_0 \in (0,1)$ be a universal constant. Let $\mathcal{P}^{(d)}_{\text{cvx}}$ be the problem instance described in \cref{sec:construction-cvx}. For every $\epsilon \leq \epsilon_0$ and $\delta \leq \epsilon$ and for every algorithm $\Alg=\{\Alg_n\}_{n \in \Naturals}$ which is a sample compression scheme of size $k$ that $\epsilon$-learns $\mathcal{P}^{(d)}_{\text{cvx}}$ with the sample complexity $\Theta\left(1/\epsilon^2\right)$ the following holds: for every $n = \Theta(1/\epsilon^2)$, and $d\geq \Omega(n^2\log(n))$, there exists a data distribution $\Dist \in \ProbMeasures{\dataspace}$ such that $k \geq \Omega(n)$.
\end{corollary}

\begin{corollary}[\textbf{Unbounded sample compression scheme in SCO, strongly convex case}]
Let $\epsilon_0$ and $\delta_0$ be universal constants.
    Let $\mathcal{P}^{(d)}_{\text{scvx}}$ be the problem instance described in \cref{sec:construction-scvx}. For every $\epsilon<\epsilon_0$ and $\delta<\delta_0$ and for every algorithm $\Alg=\{\Alg_n\}_{n \in \Naturals}$ which is a sample compression of size $k$ that $\epsilon$-learns $\mathcal{P}^{(d)}_{\text{scvx}}$ with the sample complexity $\Theta\left(1/\epsilon\right)$ the following holds: for every $n = \Theta(1/\epsilon)$, $\delta<O(1/n^2)$,  and $d\geq \Omega(n^2\log(n))$, there exists a data distribution $\Dist \in \ProbMeasures{\dataspace}$ such that $k \geq \Omega(n)$.
\end{corollary}

\subsection{Extensions to Individual Sample CMI}
One drawback of the CMI is that for many natural deterministic algorithms, it can be $\Omega(n)$.
This limitation can be attributed to the conditioning term in CMI which 
tends to reveal too much information. One notable approach to address this issue is the development of \emph{individual sample CMI} (ISCMI) in \citep{rodriguez2020random,zhou2020individually}. Consider the structure introduced in \cref{def:cmi}, then
$$
\mathrm{ISCMI}_{\Dist}(\Alg_n) \triangleq \sum_{i=1}^{n}\minf{\Alg_n(\trainset);U_i\vert Z_{0,i},Z_{1,i}}.
$$
In \citep{rodriguez2020random,zhou2020individually}, it has been shown for every learning algorithm and every data distribution $\mathrm{ISCMI}_{\Dist}(\Alg_n)  \leq \cmi$. Moreover, similar to CMI, a small ISCMI implies generalization. Therefore, it is natural to ask: \emph{Can we circumvent the lower bounds proved for CMI by measuring the information complexity of $\epsilon$-learners using $\mathrm{ISCMI}_{\Dist}(\Alg_n)$?} Our main result in this part provides a negative answer to this question. We show that exactly the same lower bound stated in \cref{thm:main-lower-convex} and \cref{thm:main-lower-stronglyconvex} holds for ISCMI. The proofs appear in \cref{appx:indiv-sample}.
\begin{corollary}[\textbf{ISCMI-accuracy tradeoff}]
\label{thm:cor-cvx-indiv}
 Let $\epsilon_0 \in (0,1)$ be a universal constant. 
Let $\mathcal{P}^{(d)}_{\text{cvx}}$ be the problem instance described in \cref{sec:construction-cvx}. For every $\epsilon \leq \epsilon_0 $ and $\delta \leq \epsilon$ and for every proper algorithm $\Alg=\{\Alg_n\}_{n \in \Naturals}$ that $\epsilon$-learns $\mathcal{P}^{(d)}_{\text{cvx}}$ with the sample complexity $N(\cdot,\cdot)$ the following holds: for every $n\geq N(\epsilon,\delta)$, and $d\geq \Omega(n^2\log(n))$, there exists a data distribution $\Dist \in \ProbMeasures{\dataspace}$ such that
\[
\nonumber
\mathrm{ISCMI}_{\Dist}(\Alg_n) \geq \Omega\left(\frac{1}{\epsilon^2}\right).
\]
\end{corollary}
\begin{corollary}[\textbf{ISCMI-accuracy tradeoff, strongly convex case}]
\label{thm:cor-scvx-indiv}
Let $\epsilon_0$ and $\delta_0$ be universal constants. Let $\mathcal{P}^{(d)}_{\text{scvx}}$ be the problem instance described in \cref{sec:construction-scvx}. For every $\epsilon<\epsilon_0$ and $\delta<\delta_0$ and for $\epsilon$-learns $\Alg$ for  $\mathcal{P}^{(d)}_{\text{scvx}}$ with the sample complexity $N(\cdot,\cdot)$ the following holds: for every $n\geq N(\epsilon,\delta)$, $\delta<O(1/n^2)$, and $d\geq \Omega(n^2\log(n))$, there exists a data distribution $\Dist \in \ProbMeasures{\dataspace}$ such that
\[
\nonumber
\mathrm{ISCMI}_{\Dist}(\Alg_n) \geq \Omega\left(\frac{1}{\epsilon}\right).
\]
\end{corollary}

\section{Overview of Characterization of CMI for the CLB SCOs}\label{sec:cvx}
In this section and in \cref{sec:scvx}, we discuss the key steps in proving the CMI lower bounds. We begin by characterizing the CMI of $\epsilon$-learners for CLB subclasses of SCOs. (All proofs are provided in \cref{appx:proof-sec-cvx}.)

For the general case when we do not impose any condition on the output of the learner, the proof turns out to be slightly more subtle.  In particular, there is a technical difference between proving the result for \emph{improper (unconstrained)} learners and \emph{proper (constrained)} learners. This issue does not appear in the strongly convex case as discussed in \cref{rem:scvx-remark}. Therefore, we begin by first proving an intermediate result for \emph{proper learners}.

\begin{remark}
    Notice that by simply scaling the problem, we can reduce the lower bound for $\clb$ with an arbitrary $L,R$ to $\mathcal{C}_{1,1}$. Therefore, for the rest of this section, we focus on $\mathcal{C}_{1,1}$. Also, without loss of generality, we can assume the parameter space is given by $\mathcal{B}_d(1)$.
\end{remark}

\subsection{Lower Bound for Proper Learners}

\subsubsection{Construction of a Hard Problem Instance for Proper Learners}
\label{sec:construction-cvx}

Let $d \in \Naturals$. Let $ \dataspace = \{\pm {1}/{\sqrt{d}} \}^d$ and $\parspace = \mathcal{B}_d(1)$. Define the loss function $\losscvx: \parspace \times \dataspace  \to \Reals$ as
\[
\nonumber
\losscvx(\theta,z) = -\inner{\theta}{z}.
\]
It is immediate to see that $\losscvx(\cdot,z)$ is $1$-Lipschitz. Let $\mathcal{P}^{(d)}_{\text{cvx}} \triangleq \SCOprob$ be the described SCO problem.

\subsubsection{Properties of $\epsilon$-Learners}
In this section, we prove several properties that are shared between every $\epsilon$-learner for $\mathcal{P}^{(d)}_{\text{cvx}}$.
\begin{lemma}
\label{lem:cvx-learner-prop}
Fix $\epsilon>0$. Let $\Alg$ be an $\epsilon$-learner for $\mathcal{P}^{(d)}_{\text{cvx}}$ with sample complexity of $N(\cdot,\cdot)$. Then, for every $\delta >0$, $n \geq N(\epsilon,\delta)$ and every $\Dist \in \ProbMeasures{\dataspace}$, with probability at least $1-\delta$, we have
$ \displaystyle \norm{\mu}-\epsilon \leq \inner{\hat{\theta}}{\mu}$, and, $\norm{\mu} - \epsilon -2\delta \leq \EE\left[\inner{\hat \theta}{\mu}\right]$ where $\hat \theta = \Alg_n(\trainset_n)$ and $\mu = \EE_{Z\sim \Dist}\left[Z\right]$.

\end{lemma}

The main implication of \cref{lem:cvx-learner-prop} is that the output of an accurate learner is positively correlated with the mean of the data distribution. As the learner does not know the data distribution, in the next result we show that the correlation to the mean of an unknown data distribution translates to a correlation between the output and the samples in the training set. The construction of the data distribution is based on the techniques developed by \citet{kamath2019privately}.
\begin{lemma}
\label{lem:fingerprinting-cvx}
Fix $\epsilon \in (0,1/12)$. For every $\epsilon$-learner $\Alg$ for $\mathcal{P}^{(d)}_{\text{cvx}}$ with sample complexity $N(\cdot,\cdot)$, there exists $\Dist \in \ProbMeasures{\dataspace}$, such that for every $\delta \in (0,1]$,
\[
\nonumber
\EE &\left[\sum_{i=1}^n \sum_{k=1}^d \left(\frac{144\epsilon^2 - d(\mu^{(k)}) ^2}{1-d(\mu^{(k)})^2}\right)\left(\hat\theta^{(k)}\right) \cdot\left(Z_i^{(k)}-\mu^{(k)}\right)\right]  \geq 6 \epsilon - 4\delta,
\]
where $n\geq N(\epsilon,\delta)$, $\trainset_n =  (Z_1,\dots,Z_n)\sim \Dist^{\otimes n}$, $\hat{\theta}=\Alg_n(\trainset)$ and $\mu = \EE_{Z\sim \Dist}[Z]$. Also, for each $k \in \range{d}$, we have $\mu^{(k)}\in [-12\epsilon/\sqrt{d},12\epsilon/\sqrt{d}]$.
\end{lemma}

\subsubsection{CMI-Accuracy Tradeoff for CLB}
\label{sec:lower-bound-proper-clb}
\begin{theorem*}[Restatement of  \cref{thm:main-lower-convex}]
Let $\epsilon_0 \in (0,1)$ be a universal constant.
    Let $\mathcal{P}^{(d)}_{\text{cvx}}$ be the problem instance described in \cref{sec:construction-cvx}. For every $\epsilon \leq \epsilon_0$ and $\delta \leq \epsilon$ and for every proper algorithm $\Alg=\{\Alg_n\}_{n \in \Naturals}$ that $\epsilon$-learns $\mathcal{P}^{(d)}_{\text{cvx}}$ with the sample complexity $N(\cdot,\cdot)$ the following holds: for every $n\geq N(\epsilon,\delta)$, and $d\geq \Omega(n^2\log(n))$, there exists a data distribution $\Dist \in \ProbMeasures{\dataspace}$ such that
\[
\nonumber
\cmi \geq \Omega\left(\frac{1}{\epsilon^2}\right).
\]
\end{theorem*}
\begin{proof}[Proof Sketch]
Let $\mathcal{P}^{(d)}_{\text{cvx}}$ be the problem instance described in \cref{sec:construction-cvx}. Fix an $\epsilon$-learner $\Alg$ for $\mathcal{P}^{(d)}_{\text{cvx}}$, and let the data distribution be such that it satisfies \cref{lem:fingerprinting-cvx}. Consider the structure introduced in the definition of CMI in \cref{def:cmi} and define diagonal matrix $A\in \Reals^{d \times d}$ where 
$
A = \mathrm{diag}\left[ \left\{\frac{144\epsilon^2 - d(\mu^{(k)})^2}{1-d(\mu^{(k)})^2}\right\}_{k=1}^{d}  \right].
$
For every $i \in \range{n}$, let $T_{0,i}=\inner{\hat{\theta}} {A\left(Z_{0,i}-\mu\right)}$, $T_{1,i}=\inner{\hat{\theta}} {A\left(Z_{1,i}-\mu\right)}$, and $\bar U_i = 1 - U_i$. Notice that $Z_{\bar U_i,i} \indep \hat{\theta}$ given $U_i$ by the definition of CMI. Then, we show that $T_{\bar{U}_i,i}$ is a sub-Gaussian random variable with a variance proxy of $O(1/\sqrt{d})$. Therefore, with a high probability,  for every $i \in \range{n}$, $\big|T_{\bar U_i,i}\big|=O(\epsilon/\sqrt{d})=O(\epsilon/n)$, since $d \geq \Omega(n^2 \log(n))$. This observation motivates us to define the set $\mathcal{I} \subseteq \range{n}$ as follows: $i\in \mathcal{I}$ if and only if $\max\{T_{1,i},T_{0,i}\}>\tau$ and $\min\{T_{1,i},T_{0,i}\}<\tau$, where $\tau = \Theta(\epsilon/n)$. 
Intuitively, elements in $\mathcal{I}$ are indices for which the output of the learner has a high correlation with the $i$-th sample in the training set and a low correlation with the corresponding ghost sample that was not observed by the learning algorithm, where $\tau$ quantifies the level of correlation.

We show that the expected cardinality of $\mathcal{I}$ is a lower bound on $\cmi$. The next step of the proof is using the fingerprinting lemma as in \cref{lem:fingerprinting-cvx} to further lower bound $|\mathcal{I}|$. We show in \cref{lem:card-moments} that we can lower bound $|\mathcal{I}|$ using the sample-wise correlation random variables. More precisely, we show that with a high probability, $|\mathcal{I}| = \Omega\left(\left(\sum_{i=1}^{n}T_{U_i,i}\right)^2/\sum_{i=1}^{n}T_{U_i,i}^2\right)$. By using \cref{lem:fingerprinting-cvx}, we show that $(\sum_{i=1}^{n}T_{U_i,i})^2=\Omega(\epsilon^2)$, and  by using \cref{lem:basis}, we show that $\sum_{i=1}^{n}T_{U_i,i}^2=O(\epsilon^4)$. Combining these two pieces concludes the proof. For a detailed proof, see \cref{appx:proof-sec-cvx}. 
\end{proof}

\subsection{Lower Bound for Improper (Unconstrained) Learners}
\label{sec:improper}
The output of proper learners is constrained into the ball of radius one in $\Reals^d$. In this section, we prove that the lower bound for improper (unconstrained) learners is reducible to the lower bound for proper (constrained) learners. Consider $\mathcal{P}^{(d)}_{\text{cvx}}=\SCOprob$ described in \cref{sec:construction-cvx}. Using $f$, we define a new loss function that is supported on $\Reals^d$ as follows: for every $z\in \dataspace$, $\tilde{f}:\Reals^d \times \dataspace \to \Reals$ is given by
\[\label{eq:improper_extension}
\tilde{f}(\theta,z)=\inf_{w\in \mathcal{B}_d(1)} \{f(w,z)+\norm{\theta - w}\}.
\]
Let $\mathcal{P}^{(d)}_{\text{cvx,improper}} = (\Theta,\dataspace,\tilde{f})$. From \cref{lem:lip-exten}, we know that $\tilde{f}(\cdot,z)$ is a $1$-Lipschitz and a convex function which means $\mathcal{P}^{(d)}_{\text{cvx,improper}} \in \mathcal{C}_{1,1}$.

\begin{theorem}
Fix $\epsilon>0$ and let $\mathcal{P}^{(d)}_{\text{cvx,improper}}$ be as described in \cref{sec:improper}. For every $\epsilon \leq 1$ and $\delta \leq \epsilon$ and for every algorithm $\Alg=\{\Alg_n\}_{n \in \Naturals}$ that $\epsilon$-learns $\mathcal{P}^{(d)}_{\text{cvx,improper}}$ with the sample complexity $N(\cdot,\cdot)$ the following holds: for every $n\geq N(\epsilon,\delta)$, and $d\geq \Omega(n^2\log(n))$, there exists a data distribution $\Dist \in \ProbMeasures{\dataspace}$ such that
\[
\nonumber
\cmi = \Omega\left(\frac{1}{\epsilon^2}\right).
\]
\end{theorem}
\begin{proof}
Let $\Alg=\{\Alg_n\}_{n\geq 1}$ be learning algorithm, possibly improper, i.e.,  $\Alg_n$ is not restricted to output an element of $\mathcal{B}_d(1)$. Also, let $\proj(\Alg)=\{\proj(\Alg)_n\}_{n\geq 1}$ as a new learning algorithm that is defined as follows: for a training set $\trainset_n \in \dataspace^n$, we have $\proj(\Alg_n)(\trainset_n)=\proj(\Alg_n(\trainset_n))$ where $\proj(\cdot):\Reals^d \to \mathcal{B}_d(1)$ is the orthogonal projection matrix onto $\mathcal{B}_d(1)$. Informally, $\proj(\Alg_n)$ is based on projecting the output $\Alg_n$ to $\mathcal{B}_d(1)$. Define $\tilde{\mathrm{F}}_{\Dist}(\theta)=\EE_{Z\sim \Dist}[\tilde{f}(\theta,Z)]$. From  \cref{lem:projection-non-increasing}, we know that  $\hat\theta =\Alg_n(\trainset_n)$ with probability one satisfies
\[
\nonumber
\tilde{\mathrm{F}}_{\Dist}(\hat \theta) - \min_{\theta \in \mathcal{B}_d(1)} \tilde{\mathrm{F}}(\theta) \geq \Popriskcvx{\proj(\hat \theta)} -\min_{\theta \in \mathcal{B}_d(1)} \Popriskcvx{\theta}.
\]
The implication of this equation is the following: if $\Alg$ is an $\epsilon$-learner for $\mathcal{P}^{(d)}_{\text{cvx,improper}}$, then, $\proj\left(\Alg\right)$ is an $\epsilon$-learner with respect to  $\mathcal{P}^{(d)}_{\text{cvx}}$.  

Notice that $\proj\left(\Alg_n\right)$ is a \emph{proper} learning algorithm. Therefore, by \cref{thm:main-lower-convex}, we have that there exists $\Dist \in \ProbMeasures{\dataspace}$ such that $\mathrm{CMI}_{\Dist}\left(\proj\left(\Alg_n\right)\right)\geq \Omega\left(\frac{1}{\epsilon^2}\right)$. Also, by \cref{lem:data-proc-proj} (data processing inequality), $\cmi \geq \mathrm{CMI}_{\Dist}\left(\proj\left(\Alg_n\right)\right) $. Ergo, for distribution $\Dist$ we also have $\cmi \geq \Omega\left(\frac{1}{\epsilon^2}\right)$.
\end{proof}

\subsection{Matching Upper Bound}

\begin{theorem}
\label{thm:lip-bounded-upperbound}
For every $L \in \Reals$, $R \in \Reals$, there exists a proper $\epsilon$-learner with sample complexity $N(\epsilon,\delta)=\frac{128 (LR)^2}{\epsilon^2}\log(2/\delta)$ such that the following holds: for every $0 <\delta\leq 1$, every $n \geq N(\epsilon,\delta)$,  every $\SCOprob \in \clb$ and every $\Dist \in \ProbMeasures{\dataspace}$ the following holds: 1) $ \Popriskcvx{\Alg(\trainset_n)}-\min_{\theta \in \parspace } \Popriskcvx{\theta} \leq \epsilon$ with probability at least $1-\delta$ and 2) $ \cmi \leq \frac{128 (LR)^2}{\epsilon^2}\log(2/\delta)$.
\end{theorem}

\section{Overview of Characterization of CMI for the CSL SCOs}
\label{sec:scvx}
In this section, we discuss the characterization of CMI of $\epsilon$-learners for CSL subclasses of SCOs. (All proofs appear in \cref{appx:proof-sec-scvx}.)

\subsection{Lower Bound}
\subsubsection{Construction of a Hard Problem Instance}

\label{sec:construction-scvx}
Towards proving \cref{thm:main-lower-stronglyconvex}, we develop the following construction: Let $d \in \Naturals$. Let $ \displaystyle \dataspace = \big\{\pm 1/\sqrt{d}\big\}^d$ and $\parspace = \Reals^d$. Define the loss function $\losscvx: \parspace \times \dataspace  \to \Reals$ as
\[
\nonumber
\losscvx(\theta,z) = -\inner{\theta}{z} + \frac{1}{2}\norm{\theta}^2.
\]
Let $\mathcal{P}^{(d)}_{\text{scvx}}\triangleq(\parspace,\dataspace,f)$ be the described problem instance.
\subsubsection{Properties of $\epsilon$-Learners for CLS}
In the next lemma, we show some properties that are shared between every $\epsilon$-learner of $\mathcal{P}^{(d)}_{\text{scvx}}$.
\begin{lemma}
\label{lem:scvx-learner-prop}
Fix $\epsilon>0$. Let $\Alg$ be an $\epsilon$-learner for $\mathcal{P}^{(d)}_{\text{scvx}}$ with the sample complexity of $N(\cdot,\cdot)$ such that its output is an element of $\mathcal{B}_d(1)$. Then, for every $\delta >0$, $n \geq N(\epsilon,\delta)$ and every $\Dist \in \ProbMeasures{\dataspace}$, with probability at least $1-\delta$, we have
$  \norm{\hat{\theta}-\mu}^2\leq 2\epsilon$,  $ \frac{1}{2}\norm{\mu}^2-\epsilon \leq \inner{\hat{\theta}}{\mu}$, and $\EE\left[\inner{\hat{\theta}}{\mu}\right] \geq \frac{1}{2}\norm{\mu}^2 - \epsilon -  \frac{3\delta}{2}. $ ,
where $\hat{\theta}=\Alg(\trainset_n)$ and $\mu = \EE_{Z \sim \Dist}\left[ Z \right]$.
\end{lemma}

\begin{remark}
\label{rem:scvx-remark}
     For learners of $\mathcal{P}_{\text{scvx}}^{(d)}$, without loss of generality, we assume that the output of the learning algorithm lies in $\mathcal{B}_d(1)$ where $\mathcal{B}_d(1)$ is the ball of radius one in $\Reals^d$. The explanation is as follows. For every $\hat{\theta} \in \Reals^d$, we have $\Popriskcvx{\hat{\theta}}-\min_{\theta \in \Reals^d}\Popriskcvx{\theta}=\frac{1}{2}\norm{\hat{\theta}-\mu}^2$. By the Pythagorean theorem we have $\norm{\proj\left(\hat{\theta}\right)-\mu}^2\leq \norm{\hat{\theta}-\mu}^2 $. Since $\mu \in \mathcal{B}_d(1)$ this shows that by projecting the output of any algorithm $\hat{\theta}$ to $\mathcal{B}_d(1)$, denoted by $\proj\left(\hat{\theta}\right)$, the excess error does not increase. Notice that projection never increases CMI due to data processing inequality \citep{cover2012elements}. Therefore, it suffices to consider the algorithms whose output lies in $\mathcal{B}_d(1)$.
\end{remark}

The next lemma is a variant of the fingerprinting lemma by \citet{steinke2016upper} which shows for a sufficiently accurate learner, there exists a distribution such that the correlation of the output and the training samples are bounded below by a constant.
\begin{lemma}
\label{lem:fingerprinting-scvx}
Fix  $\epsilon>0$. For every $\epsilon$-learner $\Alg$ for $\mathcal{P}^{(d)}_{\text{scvx}}$ with sample complexity $N(\cdot,\cdot)$, there exists a data distribution $\Dist \in \ProbMeasures{\dataspace}$ such that the following holds: for every $\delta>0$ and $n\geq N(\epsilon,\delta)$, let $\trainset_n = (Z_1,\dots,Z_n)\sim \Dist^{\otimes n}$, $\hat{\theta}=\Alg_n(\trainset_n)$ and $\mu = \EE_{Z\sim \Dist}[Z]$. Then, we have 
\[
\nonumber
\EE\left[\sum_{i=1}^{n}\inner{\hat{\theta} - \mu}{Z_i - \mu}\right] \geq \frac{1}{3} - 2\epsilon - 3\delta.
\]
\end{lemma}

\subsubsection{CMI-Accuracy Tradeoff for CSL}
\label{subsec:sketch-csl}
\begin{theorem*}[Restatement of \cref{thm:main-lower-stronglyconvex}]
Let $\epsilon_0$ and $\delta_0$ be universal constants.
    Let $\mathcal{P}^{(d)}_{\text{scvx}}$ be the problem instance described in \cref{sec:construction-scvx}. 
For every $\epsilon<\epsilon_0$ and $\delta<\delta_0$ and for every $\epsilon$-learner ($\Alg=\{\Alg_n\}_{n \in \Naturals}$), with sample complexity $N(\cdot,\cdot)$ the following holds: for every $n\geq N(\epsilon,\delta)$, $\delta<O(1/n^2)$, and $d\geq O(n^2\log(n))$, there exists a data distribution $\Dist \in \ProbMeasures{\dataspace}$ such that
\[
\nonumber
\cmi \geq \Omega\left(\frac{1}{\epsilon}\right).
\]
\end{theorem*}
\begin{proof}[Proof Sketch]
    Let $\mathcal{P}^{(d)}_{\text{scvx}}$ be the problem instance described in \cref{sec:construction-scvx}. Fix an $\epsilon$-learner $\Alg$ for $\mathcal{P}^{(d)}_{\text{scvx}}$, and let the data distribution be such that it satisfies \cref{lem:fingerprinting-scvx}. Consider the structure of the CMI introduced in \cref{def:cmi} and define for every $i \in \range{n}$, $T_{0,i}=\inner{\hat{\theta}-\mu} {Z_{0,i}-\mu}$ and $T_{1,i}=\inner{\hat{\theta}-\mu} {Z_{1,i}-\mu}$. Let $\bar U_i = 1 - U_i$. An important observation is that $Z_{\bar U_i,i} \indep \hat{\theta}$ given $U_i$. We show that $T_{\bar{U}_i,i}$ is a sub-Gaussian random variable with a variance proxy of $O(1/\sqrt{d})$. Therefore, with a high probability, for every $i \in \range{n}$, $\big|T_{\bar U_i,i}\big|=O(1/\sqrt{d})=O(1/n)$, since $d \geq \Omega(n^2 \log(n))$. This observation motivates us to define the set $\mathcal{I} \subseteq \range{n}$ as follows: $i\in \mathcal{I}$ if and only if $\max\{T_{1,i},T_{0,i}\}>\tau$ and $\min\{T_{1,i},T_{0,i}\}<\tau$, where $\tau = \Theta(1/n)$. We show that the expected cardinality of $\mathcal{I}$ is a lower bound on $\cmi$. The next step of the proof is using the fingerprinting lemma as in \cref{lem:fingerprinting-scvx} to further lower bound $|\mathcal{I}|$. Using \cref{lem:card-moments}, we show that with a high probability, $|\mathcal{I}| = \Omega\left((\sum_{i=1}^{n}T_{U_i,i})^2/\sum_{i=1}^{n}T_{U_i,i}^2\right)$. Using \cref{lem:fingerprinting-scvx}, we show $(\sum_{i=1}^{n}T_{U_i,i})^2=\Omega(1)$. Also, using \cref{lem:basis}, we show $\sum_{i=1}^{n}T_{U_i,i}^2=O(\epsilon)$. Combining these two pieces concludes the proof. For a detailed proof see \cref{appx:proof-sec-scvx}.
\end{proof}

\subsection{Matching Upper Bound}
\begin{theorem}
\label{thm:strong-cvx-upperbound}
For every $L \in \Reals$, $\mu \in \Reals$, and $\epsilon >0$, there exists an algorithm such that the following holds: for every $\SCOprob \in \scl$ and for every $n \geq \frac{2L^2}{\mu \epsilon}$, we have $\EE[\Popriskcvx{\Alg(\trainset_n)}]-\min_{\theta \in \parspace } \Popriskcvx{\theta} \leq \epsilon$,  and $\cmi \leq \frac{4L^2}{\mu \epsilon}$.
\end{theorem}

\section*{Acknowledgments}
The authors would like to thank Jonathan Ullman for insightful discussions on fingerprinting codes and privacy attacks. We also thank Konstantina Bairaktari for her valuable comments on the drafts of this work, and Sasha Voitovych for pointing out the idea of using random matrix concentration inequalities to improve \cref{lem:basis}, which led to better dimension dependence. Finally, we appreciate the ICML reviewers for their comments and suggestions, which helped improve the presentation of this paper.

\section*{Disclosure of Funding}
IA is supported by the
Vatat Scholarship from the Israeli Council for Higher Education, and the Lynn and William Frankel Center for
Computer Science at Ben-Gurion University. MH is supported by the Khoury College distinguished postdoctoral fellowship. RL is supported by a Google fellowship, and this research has been funded, in parts, by an ERC grant (FoG - 101116258). DMR is supported by an NSERC Discovery Grant and funding through his Canada CIFAR AI Chair at the Vector Institute.

\printbibliography

\newpage
\appendix

\section{Technical Lemmas}

\begin{lemma}[{\citealp[][Thm.~2.10.1]{cover2012elements}}]
\label{lem:fano}
Let $X$ and $Y$ be discrete random variables. Then
\[
\nonumber
\entr{X\vert Y} \leq \binaryentr{\proberror}+ \proberror \entr{X}\leq 1+ \proberror \entr{X},
\]
where $\proberror = \Pr(\Psi(Y)\neq X)$ for any (possibly randomized) estimator $\Psi$ of $X$ using $Y$.
\end{lemma}

\begin{lemma}[\citet{cobzas1978norm}] \label{lem:lip-exten}
Let $\mathcal{K}$ be a closed and convex subset of $\Reals^d$. Let $h: \mathcal{K}\to \Reals$ be a convex and $L$-Lipschitz function. Define  $\tilde{h}: \Reals^d \to \Reals$ as 
\[
\nonumber
\tilde{h}(x)\triangleq \inf_{y\in \mathcal{K}}\lbrace h(y) + L \norm{x-y} \rbrace.
\]
Then, we have, 1) $\tilde{h}$ is a convex and $L$-Lipschitz function, 2) for every $x \in \mathcal{K}$, $\tilde{h}(x)= h(x)$.
\end{lemma}

\begin{lemma}
\label{lem:paleyzygmund}
Let $X$ be a random variable supported on $\Reals$ with a bounded second moment. Then, for every $\theta \in \Reals$,
\[
\nonumber
\Pr\left(X \geq \theta \right) \geq
\frac{\left(\max\{\EE[X]-\theta,0\}\right)^2}{\EE[X^2]}.
\]
\end{lemma}
\begin{proof}
This is a non-standard variant of Paley-Zygmund inequality. With probability one,
\[
\nonumber
X  &= X\indic{X<\theta} + X \indic{X\geq \theta}\\
    &\leq \theta + X \indic{X\geq \theta}.
\]
Taking an expectation and using Cauchy–Schwarz inequality, we obtain
\[
\nonumber
\EE[X]\leq \theta + \sqrt{\EE[X^2]}\sqrt{\Pr\left(X\geq \theta\right)} \Rightarrow \max\{\EE[X]-\theta,0\} \leq \sqrt{\EE[X^2]}\sqrt{\Pr\left(X\geq \theta\right)},
\]
which was to be shown.
\end{proof}

\begin{lemma}
\label{lem:card-moments}
Fix $n \in \Naturals$ and $(a_1,\dots,a_n)\in \Reals^n$. Let $\sum_{i \in \range{n}}a_i = A_1$ and $\sum_{i \in \range{n}}(a_i)^2 = A_2$. Then, for every $\beta\in \Reals$,
\[
\nonumber
\big| \{i \in \range{n}~:~ a_i \geq  {\beta}/{n} \} \big| \geq \frac{\left(\max\{A_1-\beta,0\}\right)^2 }{A_2}.
\]
\end{lemma}
\begin{proof}
Define random variable $X$ with the distribution $\unif{\{a_1,\dots,a_n\}}$. By assumptions, $\EE\left[X\right]= A_1/n$ and $\EE\left[X^2\right]= A_2/n$.  Notice that
\[
\nonumber
\big| \{i \in \range{n}~:~ a_i \geq  {\beta}/{n} \} \big|
 = n \Pr\left(X\geq \beta/n\right).
\]
By \cref{lem:paleyzygmund}, we have
\[
\nonumber
\Pr\left(X \geq \beta/n\right)\geq \frac{(\max\{n\EE[X]-\beta,0\})^2}{n^2\EE[X^2]}.
\]
Therefore, 
\[
\nonumber
\big| \{i \in \range{n}~:~ a_i \geq \beta/n \} \big| &\geq \frac{(\max\{n\EE[X]-\beta,0\})^2}{n\EE[X^2]} \\
&=\frac{\left(\max\{A_1-\beta,0\}\right)^2}{A_2},
\]
as was to be shown.
\end{proof}
\begin{lemma}
\label{lem:subgaussian-randomvector}
Let $d\in \Naturals$. Let $\Dist \in \mathcal{M}_1\left(\left\{\pm {1}/{\sqrt{d}}\right\}^d\right)$ be a product distribution. Let $\mu = \EE_{Z\sim \Dist}[Z]$ and $(X_1,\dots,X_n)\sim \Dist^{\otimes n}$. Then, $\frac{1}{n}\sum_{i=1}^{n}\left(X_i - \mu\right)$ is a $\sqrt{1/(dn)}$ subguassian random vector. Moreover, 
\[
\nonumber
\Pr\left(\norm{\frac{1}{n}\sum_{i=1}^{n}X_i - \mu}^2\geq \epsilon\right) \leq 2\exp\left(\frac{-\epsilon n}{2}\right).
\]
\end{lemma}
\begin{proof}
Let $v\in \Reals^d$ be a fixed vector and $\lambda \in \Reals$ be a constant. Then,
\[
\nonumber
\EE\left[\exp\left(\frac{\lambda}{n} \sum_{i=1}^{n}\inner{\left(X_i - \mu\right)}{v}\right)\right] &= \EE\left[\prod_{i=1}^{n}\prod_{k=1}^{d}\exp\left(\frac{\lambda}{n}\left(Z_i^{(k)}-\mu^{(k)}\right)\cdot v^{(k)}\right)\right]\\
& \leq  \prod_{i=1}^{n}\prod_{k=1}^{d} \exp\left(\frac{\lambda^2 (v^{(k)})^2 }{2dn^2}\right)\\
& = \exp\left( \frac{\lambda^2 \norm{v}^2}{2dn}\right),
\]
where the second step follows from Hoeffeding's Lemma. Therefore, by definition, we have the stated result. The statement regarding the concentration of the norm follows from \citep[Lemma.~1]{jin2019short}{}.
\end{proof}

\begin{lemma}
\label{lem:expectation-prior-lb}
Fix $\beta \in [0,1]$. Let $\mu = \frac{1}{\sqrt{d}}\left(p^{(1)},\dots,p^{(d)}\right) \in \Reals^d$ where $p=\left(p^{(1)},\dots,p^{(d)}\right)$ is drawn from $\pi = \left(\unifdist[-\beta,\beta]\right)^{\otimes d}$. Then, 
\[
\nonumber
\EE\left[\norm{\mu}\right] \geq \frac{\beta}{3}.
\]
\end{lemma}
\begin{proof}
We have $\EE[(p^{(i)})^2]=\frac{\beta^2}{3}$ for every $i \in \range{d}$. Notice that $\norm{\mu}=\frac{1}{\sqrt{d}}\sqrt{\sum_{i=1}^{d}(p^{(i)})^2}$ and for every $i \in \range{d}$, $(p^{(i)})^2\in [0,\beta^2]$ with probability one. We can write 
\[
\nonumber
\norm{p}^2 = \norm{p} \norm{p} \leq \beta \sqrt{d} \norm{p}.
\]
 Therefore, we have 
\[
\nonumber
\EE[\norm{p}^2] \leq \beta \sqrt{d} \EE[\norm{p}] \Rightarrow 
 \EE[\norm{p}] \geq \frac{1}{\beta\sqrt{d}} \sum_{i=1}^d \EE[(p^{(i)})^2] = \frac{\beta}{3}\sqrt{d}.
\]
The stated result follows from $\EE[\norm{\mu}]= \frac{1}{\sqrt{d}}\EE[\norm{p}]$.
\end{proof}

\begin{lemma}
\label{lem:basis}
Let $d \in \Naturals$ and $K >0$ be a universal constant.
Let $\dataspace = \left\{\pm \frac{1}{\sqrt{d}}\right\}^d$, $\Dist \in \ProbMeasures{\dataspace}$ be a product distribution, and  $\mu = \EE_{Z\sim \Dist}[Z]$. Let $(Z_1,\dots,Z_n)\sim \Dist^{\otimes n}$ be $n$ \iid~samples. Then, for every $\beta \in (0,1]$ if $d\geq \max\{{n}/{2},{\log(2/\beta)}/{2}\}$, we have
\[
\nonumber
\Pr\left(\sup_{y \in \Reals^d} \left\{ \sum_{i=1}^{n}\left(\inner{y}{Z_i - \mu}\right)^2 - K \norm{y}^2 \right\} \leq 0\right) \geq 1-\beta.
\]
\end{lemma}
\begin{proof}

Define matrix $\mathbf{B} \in \Reals^{d \times n}$ as follows:
\[
\nonumber
\mathbf{B} = \left[Z_1,\dots,Z_d\right].
\]
In particular, the $i$-th column of $\mathbf{B}$ is $Z_i$. The main observation is that for every $y \in \Reals^d$, we have
\[
\nonumber
\sum_{i=1}^{n}\left(\inner{y}{Z_i - \mu}\right)^2 = \norm{\left[\mathbf{B}^\top  -\mu\mathbf{1}_d^\top\right]y}^2_2,
\]
where $\mathbf{1}_d$ is the all one vector of size $d$. By the definition of the operator norm, we have for every $y \in \Reals^d$ with probability one
\[
\label{eq:reduction-operator-norm}
\norm{\left[\mathbf{B}^\top  -\mathbf{1}_d \mu^\top\right]y}^2 \leq  \norm{\mathbf{B}^\top  -\mathbf{1}_d \mu^\top}^2 \norm{y}^2
\]
Consider the random matrix $\sqrt{d}\left(\mathbf{B}^\top  -\mathbf{1}_d \mu^\top\right)$. It satisfies the following two properties: 1) Its entries are \iid with zero mean and 2) each entry is bounded between $[-2,+2]$ with probability one. To argue about the operator norm of this matrix, we invoke \citep[Thm.~4.4.5]{vershynin2018high}{} to write for every $\beta\in (0,1]$, 
\[
\label{eq:operator-norm-hp}
\Pr\left(\sqrt{d}\norm{\mathbf{B}^\top  -\mathbf{1}_d \mu^\top}_2 \geq C\left(\sqrt{d}+\sqrt{n} + \sqrt{\log(2/\beta)}\right)  \right)\leq \beta,
\]
where $C$ is a universal constant.
Using \cref{eq:reduction-operator-norm} and \cref{eq:operator-norm-hp}, we conclude that 
\[
\nonumber
\Pr\left(\sup_{y \in \Reals^d}\left\{\norm{\left[\mathbf{B}^\top  -\mathbf{1}_d \mu^\top\right]y}^2 - 3C^2\left(1+\frac{n}{d} + \frac{\log(2/\beta)}{d} \right)\norm{y}^2 \right\} \leq 0 \right) \geq 1-\beta.
\]
In particular, it shows that by setting $d \geq \max\{\frac{n}{2},\frac{\log(2/\beta)}{2}\}$, we have the stated result.

\end{proof}
\begin{lemma}
\label{lem:indep-sample-corr}
Let $\dataspace = \{\pm \frac{1}{\sqrt{d}}\}^d$ and $\Dist \in \ProbMeasures{\dataspace}$ be a product measure. Define $\mu = \EE_{Z\sim \Dist}[Z]$. Then, for every fixed $y \in \Reals^d$ and $n \in \Naturals$, 
\[
\nonumber
\Pr_{(Z_1,\dots,Z_n)\sim \Dist^{\otimes n}}\left(\max_{i\in \range{n}}\{\inner{y}{Z_i-\mu}\}\geq \alpha\right) \leq n \exp\left(-\frac{\alpha^2 d}{2\norm{y}^2}\right).
\]
\end{lemma}
\begin{proof}
By union bound, $\displaystyle \Pr\left(\max_{i\in \range{n}}\{\inner{y}{Z_i-\mu}\}\geq \alpha\right) \leq n \Pr_{Z\sim \Dist}\left(\inner{y}{Z-\mu}\geq \alpha\right)$. Let $\lambda>0$ and consider
\[
\nonumber
\EE[\exp\left(\lambda \inner{y}{Z-\mu} \right)] &= \EE\left[\exp\left(\lambda \sum_{k=1}^{d}y^{(k)}\left(Z^{(k)}-\mu^{(k)} \right)\right)\right]\\
&=\prod_{k=1}^{d} \EE\left[\exp\left(\lambda y^{(k)}\left(Z^{(k)}-\mu^{(k)} \right)\right)\right]\\
&\leq \prod_{k=1}^{d} \exp\left(\lambda^2 (y^{(k)})^2 \frac{1}{2d}\right) \quad \text{(Hoeffeding's lemma since $Z^{(k)} \in \{\pm 1/\sqrt{d}\}$)}\\
&=\exp\left(\lambda^2 \norm{y}^2 \frac{1}{2d}\right).
\]
Then, using standard arguments, the stated claim can be proved.
\end{proof}

\begin{lemma}{\citep[Lemma~B.1]{shalev2014understanding}{}}
\label{lem:reverse-markov}
Let $X$ ba a non-negative random variable supported on $\Reals$ and $\Pr\left(X\leq a\right)=1$. Then, for every $\beta \in [0,a)$, we have
\[
\nonumber
\Pr\left(X> \beta \right)\geq \frac{\EE[X]-\beta}{a - \beta}.
\]
\end{lemma}

\section{Proofs for Characterization of CMI of the CLB Subclass}
\label{appx:proof-sec-cvx}

\proofsubsection{lem:cvx-learner-prop}
Notice that $\Popriskcvx{\theta} = -\inner{\theta}{\mu}$ and $ \displaystyle \min_{\theta \in \parspace} \Popriskcvx{\theta} = -\norm{\mu}$, where the minimum is achieved by setting $\theta^\star = \frac{\mu}{\norm{\mu}}$. 
Therefore, by the excess risk guarantee, with probability at least $1-\delta$,
\[
\nonumber
\Popriskcvx{\hat{\theta}} + \norm{\mu} \leq \epsilon \Rightarrow   \norm{\mu} - \epsilon \leq \inner{\hat{\theta}}{\mu}.
\]
 Notice that $\inner{\hat{\theta}}{\mu} \geq -1$, $\norm{\mu} \leq 1$, and $\epsilon > 0$, 
\[
\nonumber
\EE\left[ \inner{\hat{\theta}}{\mu} \right] &\geq \left(\norm{\mu} - \epsilon\right) \Pr\left(\inner{\hat{\theta}}{\mu} \geq \left(\norm{\mu} - \epsilon\right) \right) - \Pr\left(\inner{\hat{\theta}}{\mu} < \left(\norm{\mu} - \epsilon\right) \right)\\
& = \left(\norm{\mu} - \epsilon\right) \left(1-\Pr\left(\inner{\hat{\theta}}{\mu} < \left(\norm{\mu} - \epsilon\right) \right)\right) - \Pr\left(\inner{\hat{\theta}}{\mu} < \left(\norm{\mu} - \epsilon\right)\right)\\
& = \left(\norm{\mu} - \epsilon\right) - \Pr\left(\inner{\hat{\theta}}{\mu} < \left(\norm{\mu} - \epsilon\right)\right) \left( \norm{\mu} - \epsilon +1 \right)\\
& \geq \left(\norm{\mu} - \epsilon\right) - 2\delta, 
\]
where the last step follows because $\norm{\mu}-\epsilon+1\leq 2$ and $\Pr\left(\inner{\hat{\theta}}{\mu} < \left(\norm{\mu} - \epsilon\right)\right)\leq \delta$ by the first part of the lemma.

\proofsubsection{lem:fingerprinting-cvx}

The proof is based on defining a family of data distribution, and a prior over the family. Then, we show that in expectation over the prior, the stated claim holds. Thus, there exists a distribution with the desired property.

The data distribution is parameterized by a vector $p = \left(p^{(1)},\dots,p^{(d)}\right)\in [-1,1]^d$ where for every $z = (z^{(1)},\dots,z^{(d)}) \in \{\pm \frac{1}{\sqrt{d}}\}^d$,
\[
\nonumber
\Dist_p(z = (z^{(1)},\dots,z^{(d)}) ) = \prod_{k=1}^{d} \left( \frac{1+ \sqrt{d} z^{(k)} p^{(k)}}{2}\right).
\]
Let $\mu_p = \EE_{Z\sim \Dist_p}[Z]$ where $\mu^{(k)}_p = p^{(k)}/\sqrt{d}$ for $k \in \range{d}$.

Then we define a $\emph{prior}$ distribution $\pi \in \ProbMeasures{[-1,1]^d}$ over $p$ denoted by $\pi$ and is given by
\[
\nonumber
\pi = \unif{[-12\epsilon,12\epsilon]}^{\otimes d}.
\]
Let $\trainset_n =  (Z_1,\dots,Z_n)\sim \Dist^{\otimes n}$, and $\hat \theta  = \Alg_n(\trainset_n)$. By the same proof as presented in \citep{kamath2019privately} (see Equation 16 therein), we have that
\[
\label{eq:finger-clb}
\EE_{p \sim \pi} \EE_{\trainset_n \sim \Dist_p^{\otimes n}} \left[\sum_{i=1}^n \sum_{k=1}^d \left(\frac{144\epsilon^2 -d (\mu_p^{(k)}) ^2}{1-d(\mu_p^{(k)})^2}\right)\left(\hat\theta^{(k)}\right)\left(Z_i^{(k)}-\mu_p^{(k)}\right)\right]  = 2\EE_{p \sim \pi}\left[  
\inner{\EE_{\trainset_n \sim \Dist_p^{\otimes n}}[\hat\theta]}{\mu_p}
\right].
\]
By \cref{lem:cvx-learner-prop}, we know that for every $p\in [-1,1]^d$ 
\[
\label{eq:clb-quality}
\inner{\EE_{\trainset_n \sim \Dist_p^{\otimes n}}[\hat\theta]}{\mu_p} \geq \norm{\mu_p} - \epsilon -2\delta.
\]
Also, by \cref{lem:expectation-prior-lb}, we have
\[
\label{eq:expectation-clb-mu}
\EE_{p\sim \pi}\left[\norm{\mu_p} \right] \geq 4 \epsilon.
\]
Therefore, by \cref{eq:finger-clb,eq:clb-quality,eq:expectation-clb-mu}, we have 
\[
\nonumber
\EE_{p \sim \pi} \EE_{\trainset_n \sim \Dist_p^{\otimes n}} \left[\sum_{i=1}^n \sum_{k=1}^d \left(\frac{144\epsilon^2 - d(\mu_p^{(k)}) ^2}{1-d(\mu_p^{(k)})^2}\right)\left(\hat\theta^{(k)}\right)\cdot\left(Z_i^{(k)}-\mu_p^{(k)}\right)\right] \geq 6 \epsilon - 4\delta,
\]
which was to be shown.

\proofsubsection{thm:main-lower-convex}

Fix a learning algorithm $\Alg$ and let $\Dist$ be a distribution that satisfies \cref{lem:fingerprinting-cvx}. Also, consider the structure used in the definition of CMI in \cref{def:cmi} and let $\supersample=\{Z_{j,i}\}_{j \in \{0,1\},i \in \range{n}} \sim \Dist^{\otimes (2\times n)}$. For every $j \in \{0,1\}$ and $i \in \range{n}$, define $v_{j,i} = (v_{j,i}^{(1)},\dots,v_{j,i}^{(d)}) \in \Reals^d$ as follows.  For every $k \in [d]$, let
\[
\nonumber
v_{j,i}^{(k)} \triangleq \frac{144\epsilon^2 - d(\mu^{(k)})^2}{1-d(\mu^{(k)})^2} \left(Z_{j,i}^{(k)}-\mu^{(k)}\right).
\]
In the first step, we make the following observation. From the construction in \cref{lem:fingerprinting-cvx}, we know that $\mu^{(k)}\in [-12\epsilon/\sqrt{d},12\epsilon/\sqrt{d}]$. Simple calculations show, for $\epsilon\leq 1$, for every $k \in \range{d}$
\[
\label{eq:control-coeff}
0 \leq \frac{144\epsilon^2 - d(\mu^{(k)})^2}{1-d(\mu^{(k)})^2} \leq 144\epsilon^2.
\]

Let $\beta\triangleq\epsilon$ be a constant. Define the following set
\[
\nonumber
\mathcal{I}=\left\{(i,j) \in \range{n}\times \{0,1\} \Big| \inner{\hat{\theta}}{v_{j,i}}\geq \beta/n~\text{and}~\inner{\hat{\theta}}{v_{1-j,i}}< \beta/n\right\}.
\]
Intuitively, $\mathcal{I}$ includes the subset of columns of supersample such that one of the samples has a \emph{large} correlation with the output of the algorithm and the other one has \emph{small} correlation with the output of the algorithm.  Also, define the following event
\[
\nonumber
\mathcal{G} = \left\{\forall i \in \range{n}: \inner{\hat{\theta}}{v_{\bar{U_i},i}}<\beta/n\right\},
\]
where $\bar U_i = 1 - U_i$. 
Intuitively, under the event $\mathcal{G}$ the correlation of the output and the \emph{ghose sample} is uniformly insignificant. 

By the definition of mutual information, we can write
\[
\nonumber
\cmi &= \entr{U|\supersample} - \entr{U|\supersample,\hat{\theta}}\\
     &=    \entr{U}  - \entr{U|\supersample,\hat{\theta}}\\
     &=  n - \entr{U|\supersample,\hat{\theta}},
\]
where the second step follows from $\supersample \indep U$ and the last step follows from $\entr{U}=n$.

Notice that $\mathcal{I}$ is a  $(\hat{\theta},\supersample)$-measurable random variable, thus, $\entr{U|\supersample,\hat{\theta}} =\entr{U|\supersample,\hat{\theta},\mathcal{I}}$. Define $\mathcal{I}^{(1)}$ as follows: $i \in \mathcal{I}^{(1)}$ iff $\exists j \in \{0,1\}$ such that $(i,j)\in \mathcal{I}$. Using this notation, we can write
\begin{align} 
\entr{U|\supersample,\hat{\theta},\mathcal{I}} &= \entr{U_{\mathcal{I}^{(1)}},U_{(\mathcal{I}^{(1)})^c}|\supersample,\hat{\theta},\mathcal{I}} \nonumber\\
&\leq \entr{U_{\mathcal{I}^{(1)}} \big| \supersample,\hat{\theta},\mathcal{I}} + \entr{U_{(\mathcal{I}^{(1)})^c}|\supersample,\hat{\theta},\mathcal{I}},\label{eq:entr-decompos-cvx}
\end{align}
where the last step follows from the sub-additivity of Entropy. The second term in \cref{eq:entr-decompos-cvx} can be bounded by 
\[
\label{eq:first-term-cond-cvx}
\entr{U_{(\mathcal{I}^{(1)})^c}|\supersample,\hat{\theta},\mathcal{I}}&\leq \entr{U_{(\mathcal{I}^{(1)})^c}\big|\mathcal{I}} \\
&\leq \EE\left[(n-|\mathcal{I}|)\right],
\]
where the last step follows from $|U_{(\mathcal{I}^{(1)})^c}|\leq 2^{n - |\mathcal{I}|}$.

Define the random variable $\hat{U}\in \{0,1\}^{n}$ as follows: for every $(i,j)\in \mathcal{I}$, let $\hat{U}_i = j$. For the remaining coordinates set $\hat{U}_i=0$. Notice that $\hat{U}$ is a $\mathcal{I}$-measurable random variable. Therefore, $\entr{U_{\mathcal{I}^{(1)}} \big| \supersample,\hat{\theta},\mathcal{I}}  = \entr{U_{\mathcal{I}^{(1)}} \big| \supersample,\hat{\theta},\mathcal{I},\hat{U}} $. Then, we invoke Fano's inequality from \cref{lem:fano} to write
\[
\nonumber
\entr{U_{\mathcal{I}^{(1)}} \big| \supersample,\hat{\theta},\mathcal{I},\hat{U}} &\leq \entr{U_{\mathcal{I}^{(1)}} \big| \hat{U}} \\
& \leq  1 + \entr{U_{\mathcal{I}^c}} \Pr\left(\{\exists (i,j) \in \mathcal{I}: U_i\neq j\}\right)\\
& \leq  1 + n\Pr\left(\{\exists (i,j) \in \mathcal{I}: U_i\neq j\}\right),
\]
where the last line follows from $\entr{U_{\mathcal{I}^c}}\leq n$.

We claim that $\Pr\left(\{\exists (i,j) \in \mathcal{I}: U_i\neq j\}\right) \leq \Pr\left(\mathcal{G}^c\right)$. The proof is as follows: If there exists $(i,j) \in \mathcal{I}$ such that $U_i\neq j$, then, we have
\[
\nonumber
\inner{\hat{\theta}}{v_{\bar{U_i},i}} \geq \beta/n,
\]
by the definition of $\mathcal{I}$. Therefore, we conclude $\entr{U_{\mathcal{I}^{(1)}} \big| \supersample,\hat{\theta},\mathcal{I}} \leq 1 + n \Pr\left(\mathcal{G}^c\right)$. From \cref{eq:entr-decompos-cvx} and \cref{eq:first-term-cond-cvx}, we can write
\[
\nonumber
\entr{U \big| \supersample,\hat \theta} \leq  n - \EE\left[|\mathcal{I}|\right] + 1 + n \Pr\left(\mathcal{G}^c\right).
\]
By the definition of mutual information, we can lower bound \cmi as follows
\[
\label{eq:cardinality-cmi-cvx}
\cmi &=  n -   \entr{U|\supersample,\hat{\theta}} \\
      &\geq \EE\left[|\mathcal{I}|\right] - 1 - n \Pr\left(\mathcal{G}^c\right).
\]

In the next step of the proof, we provide a lower bound on $|\mathcal{I}|$ and $\Pr\left(\mathcal{G}^c\right)$. Under the event $\mathcal{G}$, using \cref{lem:card-moments} we can lower bound $|\mathcal{I}|$ as follows
\[
\nonumber
\EE\left[|\mathcal{I}|\right] &\geq \EE\left[|\mathcal{I}|\indic{\mathcal{G}}\right] \\
&\geq \EE\left[ \bigg|\left\{ i \in \range{n}: \inner{\hat\theta}{v_{U_i,i}} \geq \frac{\beta}{n} \right\} \bigg| \indic{\mathcal{G}}\right]\\
&\geq \EE\left[ \frac{\left(\max\left\{\sum_{i \in \range{n}}\inner{\hat\theta}{v_{U_i,i}} -\beta,0\right\}\right)^2}{\sum_{i \in \range{n}}\inner{\hat\theta}{v_{U_i,i}}^2} \indic{\mathcal{G}}\right].
\]
Define the following event
\[
\nonumber
\mathcal{E} \triangleq  \mathcal{G}   \cap \left\{\sum_{i \in \range{n}}\inner{\hat \theta}{v_{U_i,i}}^2\leq K (144\epsilon^2)^2\right\},
\]
where $K>0$ is a universal constant from \cref{lem:basis}. Since $\mathcal{E}\subseteq \mathcal{G}$, we have
\[
\nonumber
\EE\left[ \frac{\left(\max\left\{\sum_{i \in \range{n}}\inner{\hat\theta}{v_{U_i,i}} -\beta,0\right\}\right)^2}{\sum_{i \in \range{n}}\inner{\hat\theta}{v_{U_i,i}}^2} \indic{\mathcal{G}}\right] &\geq \EE\left[ \frac{\left(\max\left\{\sum_{i \in \range{n}}\inner{\hat\theta}{v_{U_i,i}} -\beta,0\right\}\right)^2}{\sum_{i \in \range{n}}\inner{\hat\theta}{v_{U_i,i}}^2} \indic{\mathcal{E}}\right]\\
&\geq \EE\left[ \frac{\left(\max\left\{\sum_{i \in \range{n}}\inner{\hat\theta}{v_{U_i,i}} -\beta,0\right\}\right)^2}{K (144\epsilon^2)^2} \indic{\mathcal{E}}\right],
\]
where the last step follows because under the event $\mathcal{E}$, $\sum_{i \in \range{n}}\inner{\hat\theta}{v_{U_i,i}}^2 \leq K (144\epsilon^2)^2$. Then,   
\[
\label{eq:expect-decompose-clb}
\EE\left[ \frac{\left(\max\left\{\sum_{i \in \range{n}}\inner{\hat\theta}{v_{U_i,i}} -\beta,0\right\}\right)^2}{K (144\epsilon^2)^2} \indic{\mathcal{E}}\right] &= \EE\left[ \frac{\left(\max\left\{\sum_{i \in \range{n}}\inner{\hat\theta}{v_{U_i,i}} -\beta,0\right\}\right)^2}{K (144\epsilon^2)^2} \right] \\
&- \EE\left[ \frac{\left(\max\left\{\sum_{i \in \range{n}}\inner{\hat\theta}{v_{U_i,i}} -\beta,0\right\}\right)^2}{K (144\epsilon^2)^2} \indic{\mathcal{E}^c}\right].
\]
The first term in \cref{eq:expect-decompose-clb} can be lower bounded as 
\[
\EE\left[ \frac{\left(\max\{\sum_{i \in \range{n}}\inner{\hat\theta}{v_{U_i,i}} -\beta,0\}\right)^2}{K (144\epsilon^2)^2} \right]  &\geq \frac{\left(\max\{\EE\left[\sum_{i \in \range{n}}\inner{\hat\theta}{v_{U_i,i}}\right] -\beta,0\}\right)^2}{K (144\epsilon^2)^2}\\
&\geq \frac{\left(\max\{6\epsilon - 4 \delta -\beta,0\}\right)^2}{K (144\epsilon^2)^2}.
\]
where the first step follows from convexity of $h_1(x)=x^2$ and $h_2(x)=\max\{x,0\}$ and applying Jensen's inequality. The second step follows from \cref{lem:fingerprinting-cvx}. Since $\delta<\epsilon$ and $\beta=\epsilon$, 
\[
\nonumber
\EE\left[ \frac{\left(\max\left\{\sum_{i \in \range{n}}\inner{\hat\theta}{v_{U_i,i}} -\beta,0\right\}\right)^2}{ K(144\epsilon^2)^2} \right]  = \Omega\left(\frac{1}{\epsilon^2}\right).
\]
The second term in \cref{eq:expect-decompose-clb} can be upper bounded by
\[
\nonumber
\EE\left[ \frac{\left(\max\{\sum_{i \in \range{n}}\inner{\hat\theta}{v_{U_i,i}} -\beta,0\}\right)^2}{K(144\epsilon^2)^2} \indic{\mathcal{E}^c}\right]\leq  \frac{O\left(\epsilon^4 n^2 +\epsilon^2\right)}{K(144\epsilon^2)^2}\cdot\Pr\left(\mathcal{E}^c\right),
\]
where the last step follows from
\[
\nonumber
\left(\max\left\{\sum_{i \in \range{n}}\inner{\hat\theta}{v_{U_i,i}} -\beta,0\right\}\right)^2 &\leq 2 \norm{\hat \theta}^2\norm{\sum_{i \in \range{n}}v_{U_i,i}}^2 + 2\beta^2\\
& = O\left(\epsilon^4 n^2\right)+ 2\beta^2\\
&= O\left(\epsilon^4 n^2 + \epsilon^2\right).
\]
To see the second step,  define the diagonal matrix $A\in \Reals^{d \times d}$ as 
\[
\nonumber
A = \mathrm{diag}\left[ \left\{\frac{144\epsilon^2 - d(\mu^{(k)})^2}{1-d(\mu^{(k)})^2}\right\}_{k=1}^{d}  \right].
\]
Note that from  \cref{eq:control-coeff}, we have $\norm{A}_2\leq 144\epsilon^2$. Therefore, we have
\[
\nonumber
\norm{v_{U_i,i}} &= \norm{A (Z_{U_i,i} - \mu)} \\
&\leq \norm{A}_2 \norm{Z_{U_i,i} - \mu}\\
&\leq 288\epsilon^2.
\]

In the last step, we need to show that for sufficiently small $\gamma$, $\Pr\left(\mathcal{E}^c\right)\leq \frac{\gamma}{n^2}$. By the definition of event $\mathcal{E}$, we can use union bound to write
$$ 
 \Pr\left(\mathcal{E}^c\right)\leq \Pr\left(\mathcal{G}^c\right) + \Pr\left(\sum_{i\in \range{n}}\inner{\hat \theta}{v_{U_i,i}}> K\left(144\epsilon^2\right)^2 \right).
$$
Notice that
\[
\nonumber
\Pr\left(\mathcal{G}^c\right)  &= \Pr\left(\max_{i \in \range{n}}\left\{\inner{\hat \theta}{v_{\bar U_i,i}}\right\}\geq \beta/n\right)\\
&=\EE\left[\Pr\left(\max_{i \in \range{n}}\left\{\inner{\hat \theta}{A \left(Z_{\bar U_i,i} - \mu\right)}\right\}\geq \beta/n \bigg| U,\hat \theta\right)\right]\\
&=\EE\left[\Pr\left(\max_{i \in \range{n}}\left\{\inner{A\hat \theta}{ Z_{\bar U_i,i} - \mu}\right\}\geq \beta/n \bigg| U,\hat \theta\right)\right],
\]
where the last step follows because $A$ is a diagonal matrix. Conditioned on $U$ and $\hat{\theta}$ and $Z_{\bar U_i,i}$ are independent by the construction of CMI in \cref{def:cmi}. This observation lets us use \cref{lem:indep-sample-corr} to upper bound the probability inside the expectation:
\[
\nonumber
\Pr\left(\max_{i \in \range{n}}\left\{\inner{A\hat \theta}{ Z_{\bar U_i,i} - \mu}\right\}\geq \beta/n \bigg| U,\hat \theta\right) &\leq n \exp\left( - \frac{d \epsilon^2}{n^2 \cdot (144 \epsilon^2)^2}\right) \\
&\leq n \exp\left( - \frac{d }{n^2 \cdot (144)^2}\right),
\]
where the first step follows from $\norm{A\hat \theta}\leq \norm{A}\norm{\hat \theta}\leq \norm{A}\leq 144 \epsilon^2$ and the second step follows because $\epsilon\leq 1$. Therefore, setting $d\geq \Omega(n^2 \log(n^3))$, we have
\[
\label{eq:uncorrelated-complement-prob}
\Pr\left(\mathcal{G}^c\right) \leq O\left(\frac{1}{n^2}\right).
\]
By \cref{eq:control-coeff}, $ \norm{A}_2\leq 144\epsilon^2$. Since $A$ is a diagonal matrix, we can write
\[
\nonumber
\inner{\hat \theta}{v_{U_i,i}} = \inner{A \hat \theta}{\left(Z_{U_i,i} - \mu\right)}.
\]
By $\norm{\hat \theta}\leq 1$, we have $\norm{A\hat\theta}^2\leq \norm{A}^2\norm{\hat \theta}^2\leq (144 \epsilon^2)^2$. Therefore, we can write
\[
\nonumber
\Pr\left(\sum_{i=1}^{n} \inner{\hat \theta}{v_{U_i,i}}^2 \geq K (144 \epsilon^2)^2\right) &= \Pr\left(\sum_{i=1}^{n} \inner{A \hat \theta}{Z_{U_i,i}-\mu}^2 \geq K (144 \epsilon^2)^2\right)\\
&\leq \Pr\left(\sum_{i=1}^{n} \inner{A \hat \theta}{Z_{U_i,i}-\mu}^2 \geq K\norm{A\hat \theta}^2\right)\\
&\leq \EE\left[\Pr\left(\sum_{i=1}^{n} \inner{A \hat \theta}{Z_{U_i,i}-\mu}^2 \geq K\norm{A\hat \theta}^2 \bigg | U\right)\right].
\]

Using this representation, we can use \cref{lem:basis} to conclude that given $d>\Omega\left(n \log(n)\right)$
\[
\label{eq:basis-complement-prob}
\Pr\left(\sum_{i=1}^{n} \inner{A \hat \theta}{Z_{U_i,i}-\mu}^2 \geq K\norm{A\hat \theta}^2 \bigg | U\right) \leq O\left(\frac{1}{n^2}\right).
\]
To conclude this step, \cref{eq:uncorrelated-complement-prob} and \cref{eq:basis-complement-prob} show
\[
\nonumber
\Pr\left(\mathcal{E}^c\right)&\leq  \Pr\left(\mathcal{G}^c\right) + \Pr\left(\sum_{i\in \range{n}}\inner{\hat \theta}{v_{U_i,i}}> K\left(144\epsilon^2\right)^2 \right)\\
&\leq O\left(\frac{1}{n^2}\right).
\]
We showed that $\EE[\mathcal{I}]=\Omega(1/\epsilon^2)$ and $\Pr\left(\mathcal{G}^c\right)=O(1/n^2)$. Therefore, using \cref{eq:cardinality-cmi-cvx}, we obtain
\[
\nonumber
\cmi &\geq \EE\left[|\mathcal{I}|\right] - 1 - n \Pr\left(\mathcal{G}^c\right) \\
&\geq \Omega\left(1/\epsilon^2\right),
\]
as was to be shown.

\subsection{Corollaries of Proof of \cref{thm:main-lower-convex}}

\begin{corollary}\label{cor:num-cor-cvx}
    Let $\mathcal{P}^{(d)}_{\text{cvx}}$ be the problem instance described in \cref{sec:construction-cvx}. 
Fix $\epsilon<1$. 
For every $\delta \leq \epsilon$ and for every algorithm $\Alg=\{\Alg_n\}_{n \in \Naturals}$ that $\epsilon$-learns $\mathcal{P}^{(d)}_{\text{cvx}}$ with the sample complexity $N(\cdot,\cdot)$ the following holds: for every $n\geq N(\epsilon,\delta)$, and $d\geq \Omega(n^2\log(n))$, there exists a data distribution $\Dist \in \ProbMeasures{\dataspace}$ such that
\[
\nonumber
\EE\left[\Big| \left\{i \in \range{n}: \inner{\hat \theta }{A\left(Z_i - \mu\right)} \right\} \geq \frac{\epsilon}{n} \Big|\right] = \Omega\left(\frac{1}{\epsilon^2}\right),
\]
where $\trainset_n = (Z_1,\dots,Z_n)\sim \Dist^{\otimes n}$, $\hat \theta = \Alg(\trainset_n)$, and $\mu = \EE_{Z\sim \Dist}\left[Z\right]$, and 
\[
\nonumber
A = \mathrm{diag}\left[ \left\{\frac{144\epsilon^2 - d(\mu^{(k)})^2}{1-d(\mu^{(k)})^2}\right\}_{k=1}^{d}  \right].
\]
\end{corollary}

\begin{corollary} \label{cor:cardI-cvx}
Fix $\epsilon\in (0,1)$.  Consider the structure introduced in the definition of CMI in \cref{def:cmi}. Then, define the random set
\[
\nonumber
\mathcal{I}=\left\{(i,j) \in \range{n}\times \{0,1\} \Big| \inner{\hat \theta}{A\left(Z_{j,i}-\mu\right)}\geq \epsilon/n~\text{and}~\inner{\hat \theta}{A\left(Z_{1-j,i}-\mu\right)}< \epsilon/n\right\},
\]
where
$A = \mathrm{diag}\left[ \left\{\frac{144\epsilon^2 - d(\mu^{(k)})^2}{1-d(\mu^{(k)})^2}\right\}_{k=1}^{d}  \right]$, 
 $\trainset_n = (Z_{U_1,1},\dots,Z_{U_n,n})$, $\hat \theta = \Alg(\trainset_n)$, and $\mu = \EE_{Z\sim \Dist}\left[Z\right]$. 
 
 Let $\mathcal{P}^{(d)}_{\text{cvx}}$ be the problem instance described in \cref{sec:construction-cvx}. 
For every $\delta \leq \epsilon$ and for every algorithm $\Alg=\{\Alg_n\}_{n \in \Naturals}$ that $\epsilon$-learns $\mathcal{P}^{(d)}_{\text{cvx}}$ with the sample complexity $N(\cdot,\cdot)$ the following holds: for every $n\geq N(\epsilon,\delta)$, and $d\geq \Omega(n^2\log(n))$, there exists a data distribution $\Dist \in \ProbMeasures{\dataspace}$ such that 
\[
\nonumber
\EE\left[|\mathcal{I}|\right]=\Omega\left(\frac{1}{\epsilon^2}\right).
\]
\end{corollary}

\proofsubsection{thm:lip-bounded-upperbound}
 Given that the Euclidean radius of $\parspace$ is bounded by $R$, we will presume that the loss function lies within $[-LR,LR]$.  Let $0< m \leq n$ and $\eta>0$ be constants which are determined later. The algorithm $\Alg_n$ is based on early-stopped online gradient descent. More precisely, let the training set $\trainset_n = (Z_1,\dots,Z_n)$ and $\theta_1=0$. For $t \in \range{m}$, let 
\[
\nonumber
\theta_{t+1} = \proj_{\parspace}\left( \theta_t - \eta \partial  \losscvx(\theta_t, Z_t)\right),
\]
where $\partial  \losscvx(\theta_t, Z_t)$ denotes the sub-gradient of $\partial  \losscvx(\cdot, Z_t)$ at $\theta_t$. Then, the output of the algorithm will be $\Alg_n(\trainset_n)=\frac{1}{m}\sum_{t=1}^{m}\theta_t$.

 By the standard result on the regret analysis of the online gradient descent and the online-to-batch conversion in \citep{zinkevich2003online,shalev2009stochastic,orabona2019modern}, we have with probability at least $1-\delta$, 
\[
\nonumber
\Popriskcvx{\Alg_n(\trainset_n)}- \min_{\theta \in \parspace}\Popriskcvx{\theta} \leq \frac{R^2}{2m\eta} + \frac{\eta}{2}L^2 + 2LR \sqrt{\frac{8 \log(2/\delta)}{m}}
\]
By setting  $\displaystyle m = 128\frac{(LR)^2}{\epsilon^2} \log(2/\delta)$ and $ \displaystyle \eta = \frac{R}{L}\frac{1}{\sqrt{m}}$, $\Alg_n$ achieves $\epsilon$ excess risk of $\epsilon$ with probability at least $1-\delta$.
Next, we provide the analysis of CMI of $\Alg_n$.  Using the chain rule for mutual information, we have
\[
\nonumber
\cmi &= \minf{\Alg_n(\trainset_n);U\vert \supersample}\\
    &= \minf{\Alg_n(\trainset_n);U_{1},\dots,U_{n}\vert \supersample}\\
    & = \minf{\Alg_n(\trainset_n);U_{1},\dots,U_{m}\vert \supersample} + \minf{\Alg_n(\trainset_n);U_{m+1},\dots,U_{n}\vert \supersample,U_{1},\dots,U_{m}}.
\]
Since $\Alg_n(\trainset_n)$ depends only on the first $m$ examples in the training set, $\minf{\Alg_n(\trainset_n);U_{m+1},\dots,U_{n}\vert \supersample,U_{1},\dots,U_{m}}=0$. Therefore,
\[
\cmi &=\minf{\Alg_n(\trainset);U_{1},\dots,U_{m}\vert \supersample}\\
    &\leq \entr{U_1,\dots,U_{m} \vert \supersample}\\
    &= \entr{U_{1},\dots,U_{m}}\\
    &\leq m.
\]
 Therefore its CMI is less than $m$ as was to be shown.

\section{Auxiliary Lemma for Improper Learning of the CLB Subclass}

\begin{lemma}
\label{lem:projection-non-increasing}
Let $\mathcal{B}_d(1)$ denote the ball of radius one in $\Reals^d$. Let $f:\mathcal{B}_d(1) \times \dataspace \to \Reals $ be a convex and $1$-Lipschitz loss function defined over $\mathcal{B}_d(1)$. Then, there exists a convex and $1$-Lipschitz $\tilde{f}: \Reals^d \times \dataspace \to \Reals$ such that for every  $\hat{\theta} \in \Reals^d$ and every $\Dist$, we have
\[
\nonumber
\EE_{Z \sim \Dist}\left[ \tilde{f}(\hat{\theta},Z)\right] - \min_{\theta \in \mathcal{B}_d(1)} \EE_{Z \sim \Dist}\left[ \tilde{f}(\theta,Z)\right] \geq \EE_{Z \sim \Dist}\left[ f(\proj\left(\hat{\theta}\right),Z)\right] - \min_{\theta \in \mathcal{B}_d(1)} \EE_{Z \sim \Dist}\left[ f(\theta,Z)\right],
\]
where $\proj(\cdot):\Reals^d \to \mathcal{B}_d(1)$ is the orthogonal projection operator on $\mathcal{B}_d(1)$.

\end{lemma}

\begin{proof}
Let $f: \mathcal{B}_d(1) \times \dataspace \to \Reals$ be a convex and $1$-Lipschitz loss function. For every $z\in \dataspace$, define 
\[
\nonumber
\tilde{f}(\theta,z) = \inf_{w \in \mathcal{B}_d(1)} \lbrace f(w,z) + \norm{\theta-w} \rbrace.
\]
By \cref{lem:lip-exten}, we know that for every $z \in \dataspace$, $\tilde{f}(\cdot,z)$ is convex and $1-$Lipschitz. Our first claim is that 
$$
\min_{\theta \in \mathcal{B}_d(1)} \EE_{Z \sim \Dist}\left[ \tilde{f}(\theta,Z)\right] = \min_{\theta \in \mathcal{B}_d(1)} \EE_{Z \sim \Dist}\left[  f(\theta,Z) \right].
$$
It follows from the fact that for every $\theta \in \mathcal{B}_d(1)$ and every $z\in \dataspace$, $\tilde{f}(\theta,z)=f(\theta,z)$ by \cref{lem:lip-exten}. Let $\proj:\Reals^d \to \mathcal{B}_d(1)$ be the projection operator. Our second claim is that for every $\theta \in \Reals^d$, we have
\[
\EE_{Z \sim \Dist}\left[ \tilde{f}(\proj\left(\theta\right),Z)\right] \leq  \EE_{Z \sim \Dist}\left[  \tilde{f}(\theta,Z) \right].
\]
The proof is as follows. For every $z \in \dataspace$, we can write
\[
\nonumber
\tilde{f}\left(\theta,z\right) &= \inf_{w\in \mathcal{B}_d(1)} \lbrace f(w,z) + \norm{\theta - w}\rbrace\\
& \geq \inf_{w\in \mathcal{B}_d(1)} \lbrace f(w,z) + \norm{\proj\left(\theta\right) - w}\rbrace,
\]
where the last step follows from 
$$
\norm{\theta - w}\geq \norm{\Pi(\theta) - \Pi(w)}=\norm{\Pi(\theta) - w}
$$ 
where the first step is by contraction property of the projection and the second step is due to $\Pi(w)=w$ since $w\in \mathcal{B}_d(1)$. Then, notice that 
\[
\nonumber
\inf_{w\in \mathcal{B}_d(1)} \lbrace f(w,z) + \norm{\proj\left(\theta\right) - w}\rbrace &= \tilde{f}(\proj\left(\theta\right),z) \\
& = f(\proj\left(\theta\right),z).
\]
The last step follows from $\proj(\theta) \in \mathcal{B}_d(1)$ and by \cref{lem:lip-exten}, $\tilde{f}(.,z)$ and $f(.,z)$ agree on $\mathcal{B}_d(1)$. 

Combining these two claims we obtain, for every $\hat{\theta} \in \Reals^d$, we have
\[
\nonumber
\EE_{Z \sim \Dist}\left[ \tilde{f}(\hat{\theta},Z)\right] - \min_{\theta \in \mathcal{B}_d(1)} \EE_{Z \sim \Dist}\left[ \tilde{f}(\theta,Z)\right] \geq \EE_{Z \sim \Dist}\left[ f(\proj\left(\hat{\theta}\right),Z)\right] - \min_{\theta \in \mathcal{B}_d(1)} \EE_{Z \sim \Dist}\left[ f(\theta,Z)\right],
\]
as was to be shown.
\end{proof}

\begin{lemma}
\label{lem:data-proc-proj}
Let $\Alg_n$ be a learning algorithm. Define $\proj(\Alg_n)$ as a learning algorithm that obtains by projecting the output of $\Alg_n$ into $\mathcal{B}_d(1)$. Then,
\[
\nonumber
\mathrm{CMI}_{\Dist}(\Alg_n) \geq \mathrm{CMI}_{\Dist}(\proj(\Alg_n))  
\]
\end{lemma}
\begin{proof}
    This result is a direct corollary of the data processing inequality \citep{cover2012elements}. 
\end{proof}

\section{Proofs for Characterization of CMI of the CSL Subclass}
\label{appx:proof-sec-scvx}

\proofsubsection{lem:scvx-learner-prop}

For every $\theta \in \Reals^d$, we have $ \displaystyle
\Popriskcvx{\theta} = -\inner{\theta}{\mu} + \frac{1}{2}\norm{\theta}^2$ ,
and $ \displaystyle \min_{\theta \in \parspace} \Popriskcvx{\theta} = \frac{-1}{2}\norm{\mu}^2$ where the minimum is achieved by setting $\theta^\star = \mu$. Therefore, a simple calculation shows that
\[
\nonumber
 \Popriskcvx{\theta} -  \Popriskcvx{\theta^\star} &= \frac{1}{2}\norm{\theta - \mu}^2 \\
 &= \frac{1}{2}\norm{\theta}^2 - \inner{\theta}{\mu} + \frac{1}{2}\norm{\mu}^2\\
 &\geq\frac{1}{2}\norm{\mu}^2 - \inner{\theta}{\mu}.
\]
Thus, if $\hat{\theta}$ achieves excess error $\epsilon$ with probability at least $1-\delta$, we have $\frac{1}{2} \norm{\mu}^2 - \inner{\hat{\theta}}{\mu} \leq \epsilon$ and $\norm{\theta - \mu}^2\leq 2\epsilon$.

For the in-expectation result, notice that without loss of generality, we can assume that $\hat{\theta} \in \mathcal{B}_d(1)$.  

\proofsubsection{lem:fingerprinting-scvx}

The proof is based on defining a family of data distribution, and a prior over the family. Then, we show that in expectation over the prior, the stated claim holds. 

The data distribution is parameterized by a vector $p = \left(p^{(1)},\dots,p^{(d)}\right)\in [-1,1]^d$ where for every $z = (z^{(1)},\dots,z^{(d)}) \in \{\pm \frac{1}{\sqrt{d}}\}^d$,
\[
\nonumber
\Dist_p(z = (z^{(1)},\dots,z^{(d)}) ) = \prod_{k=1}^{d} \left( \frac{1+ \sqrt{d} z^{(k)} p^{(k)}}{2}\right).
\]
Let $\mu_p = \EE_{Z\sim \Dist_p}[Z]$ where $\mu^{(k)}_p = p^{(k)}/\sqrt{d}$. We define a $\emph{prior}$ distribution $\pi \in \ProbMeasures{[-1,1]^d}$ over $p$ as follows
\[
\nonumber
\pi = \unif{[-1,1]}^{\otimes d}.
\]
Let $\trainset_n =  (Z_1,\dots,Z_n)\sim \Dist_p^{\otimes n}$, and $\hat \theta  = \Alg_n(\trainset_n)$. From Lemmas  4.3.7 and  4.3.8 of \citep{steinke2016upper}, we have the following result known as fingerprinting lemma:
\[
\nonumber
&\EE_{p\sim \pi}\EE_{\trainset_n\sim \Dist^{\otimes n}}\left[ \inner{\hat{\theta}}{\sum_{i \in \range{n}}\left( Z_i - \mu_p  \right)} \right] \\
&= 2\EE_{p\sim \pi}\left[\inner{\EE_{\trainset_n\sim \Dist^{\otimes n}}[\hat \theta]}{\mu_p}\right].
\]
By \cref{lem:scvx-learner-prop}, for every $p$
\[
\nonumber
\inner{\EE_{\trainset_n\sim \Dist_p^{\otimes n}}[\hat \theta]}{\mu_p} \geq \frac{\norm{\mu_p}^2}{2}-\epsilon - \frac{3}{2}\delta.
\]
Therefore, 
\[
\nonumber
\EE_{p\sim \pi}\inner{\EE_{\trainset_n\sim \Dist_p^{\otimes n}}[\hat \theta]}{\mu_p} &\geq \EE_{p\sim \pi}\left[\frac{\norm{\mu_p}^2}{2}\right]-\epsilon - \frac{3}{2}\delta\\
&=\frac{1}{6}-\epsilon - \frac{3}{2}\delta,
\]
where the last step follows from
\[
\nonumber
\EE_{p\sim \pi}\left[\norm{\mu_p}^2\right] = \sum_{k=1}^{d}\frac{1}{d} \EE_{p\sim \pi}\left[(p^{(k)})^2\right]=\frac{1}{3}.
\]
Therefore, 
\[
\nonumber
\EE_{p\sim \pi}\EE_{\trainset_n\sim \Dist_p^{\otimes n}}\left[ \inner{\hat{\theta}}{\sum_{i \in \range{n}}\left( Z_i - \mu_p  \right)} \right] &= 2\EE_{p\sim \pi}\left[\inner{\EE_{\trainset_n\sim \Dist_p^{\otimes n}}[\hat \theta]}{\mu_p}\right] \\
&\geq \frac{1}{3}-2\epsilon - 3\delta,
\]
as was to be shown.

\proofsubsection{thm:main-lower-stronglyconvex}

Fix a learning algorithm $\Alg$, and let $\Dist$ be a distribution satisfies \cref{lem:fingerprinting-scvx}. Also, consider the structure introduced in the definition of CMI in \cref{def:cmi} and let $\supersample=(Z_{j,i})_{j \in \{0,1\},i \in \range{n}} \sim \Dist^{\otimes (2\times n)}$. Let $\beta=1/12$ be a constant. Define the following set
\[
\nonumber
\mathcal{I}=\left\{(i,j) \in \range{n}\times \{0,1\} \Big| \inner{\hat{\theta}-\mu}{Z_{j,i}-\mu}\geq \beta/n~\text{and}~\inner{\hat{\theta}-\mu}{Z_{1-j,i}-\mu}< \beta/n\right\}.
\]
Intuitively, $\mathcal{I}$ includes the subset of columns of supersample such that one of the samples has a \emph{large} correlation to the output of the algorithm and  the other one has \emph{small} correlation to the output of the algorithm.  Also, define the following event
\[
\nonumber
\mathcal{G} = \left\{\forall i \in \range{n}: \inner{\hat{\theta}-\mu}{Z_{\bar{U_i},i}-\mu}<\beta/n\right\},
\]
where $\bar U_i = 1 - U_i$. 
Intuitively, under the event $\mathcal{G}$ the correlation of the output and the \emph{ghost samples} are insignificant. We can write
\[
\nonumber
\cmi &= \entr{U|\supersample} - \entr{U|\supersample,\hat{\theta}}\\
     &=    \entr{U}  - \entr{U|\supersample,\hat{\theta}}\\
     &=  n - \entr{U|\supersample,\hat{\theta}}.
\]
where the last two steps follows from  $U \indep \supersample$ and $\entr{U}=n$. 

Notice that $\mathcal{I}$ is a  $(\hat{\theta},\supersample)$-measurable random variable, thus, $\entr{U|\supersample,\hat{\theta}} =\entr{U|\supersample,\hat{\theta},\mathcal{I}}$. Define $\mathcal{I}^{(1)}$ as follows: $i \in \mathcal{I}^{(1)}$ iff $\exists j \in \{0,1\}$ such that $(i,j)\in \mathcal{I}$. Using this notation, we can write
\begin{align} 
\entr{U|\supersample,\hat{\theta},\mathcal{I}} &= \entr{U_{\mathcal{I}^{(1)}},U_{(\mathcal{I}^{(1)})^c}|\supersample,\hat{\theta},\mathcal{I}} \nonumber\\
&\leq \entr{U_{\mathcal{I}^{(1)}} \big| \supersample,\hat{\theta},\mathcal{I}} + \entr{U_{(\mathcal{I}^{(1)})^c}|\supersample,\hat{\theta},\mathcal{I}},\label{eq:entr-decompos-scvx}
\end{align}
where the last step follows from sub-additivity of the discrete Entropy. The second term in \cref{eq:entr-decompos-scvx} can be bounded by 
\[
\nonumber
\entr{U_{(\mathcal{I}^{(1)})^c}|\supersample,\hat{\theta},\mathcal{I}} &\leq \entr{U_{(\mathcal{I}^{(1)})^c}|\mathcal{I}}  \\
&\leq \EE\left[(n-|\mathcal{I}|)\right],
\]
where the last step follows because of the cardinality bound on the discrete entropy.

Define the random variable $\hat{U}\in \{0,1\}^{n}$ as follows: for every $(i,j)\in \mathcal{I}$, let $\hat{U}_i = j$. For the remaining coordinates set $\hat{U}_i=0$. Notice that $\hat{U}$ is a $\mathcal{I}$ measurable random variable. Therefore, $\entr{U_{\mathcal{I}^{(1)}} \big| \supersample,\hat{\theta},\mathcal{I}}  = \entr{U_{\mathcal{I}^{(1)}} \big| \supersample,\hat{\theta},\mathcal{I},\hat{U}} $. Then, we invoke Fano's inequality from \cref{lem:fano} to write
\[
\nonumber
\entr{U_{\mathcal{I}^{(1)}} \big| \supersample,\hat{\theta},\mathcal{I},\hat{U}} &\leq \entr{U_{\mathcal{I}^{(1)}} \big| \hat{U}} \\
& \leq  1 + \entr{U_{\mathcal{I}^{(1)}}} \Pr\left(\{\exists (i,j) \in \mathcal{I}: U_i\neq j\}\right)\\
& \leq  1 + n\Pr\left(\{\exists (i,j) \in \mathcal{I}: U_i\neq j\}\right),
\]
where the first step follows because conditioning never increase the entropy, the second step follows from Fano's inequality, and the third step follows from the cardinality bound $\entr{U_{\mathcal{I}^{(1)}}}$.

We claim that $\Pr\left(\{\exists (i,j) \in \mathcal{I}: U_i\neq j\}\right) \leq \Pr\left(\mathcal{G}^c\right)$. The proof is as follows: If there exists $(i,j) \in \mathcal{I}$ such that $U_i\neq j$, then, we have
\[
\nonumber
\inner{\hat \theta - \mu}{ Z_{\bar{U}_i,i}-\mu}\geq \frac{\beta}{n},
\]
by the definition of $\mathcal{I}$. Therefore, we conclude $\entr{U_{\mathcal{I}^{(1)}} \big| \supersample,\hat{\theta},\mathcal{I}} \leq 1 + n \Pr\left(\mathcal{G}^c\right)$. The conditional entropy can be upper bounded by
\[
\nonumber
\entr{U \big| \supersample,\hat \theta} \leq  n - \EE\left[|\mathcal{I}|\right] + 1 + n \Pr\left(\mathcal{G}^c\right).
\]
By the definition of mutual information, we can lower bound \cmi as follows
\[
\label{eq:cardinality-cmi-scvx}
\cmi &=  n -   \entr{U|\supersample,\hat{\theta}} \\
      &\geq \EE\left[|\mathcal{I}|\right] - 1 - n \Pr\left(\mathcal{G}^c\right).
\]
In the next step of the proof, we provide a lower bound on $\EE\left[|\mathcal{I}|\right]$. Let us define a random variable that measures the \emph{correlation} between the output and the $i$-th training samples:
\[
\nonumber
c_i \triangleq \inner{\hat{\theta}-\mu}{Z_{U_i,i}-\mu}.
\]
Under the event $\mathcal{G}$, using \cref{lem:card-moments} we can lower bound $\EE[|\mathcal{I}|]$ as follows
\[
\label{eq:cardinality-cvx-main}
\EE[|\mathcal{I}|]&\geq \EE\left[|\mathcal{I}|\indic{\mathcal{G}}\right] \\
&\geq \EE\left[ \bigg|\left\{ i \in \range{n}: c_i \geq \frac{\beta}{n} \right\} \bigg| \indic{\mathcal{G}}\right]\\
&\geq \EE\left[ \frac{\left(\max\{\sum_{i \in \range{n}}c_i -\beta,0\}\right)^2}{\sum_{i \in \range{n}}c_i^2} \indic{\mathcal{G}}\right].
\]
Also, define the following event
\[
\nonumber
\mathcal{E} \triangleq  \mathcal{G}  \cap \left\{\norm{\hat{\theta}-\mu}^2\leq \epsilon\right\} \cap \left\{\sum_{i \in \range{n}}c_i^2\leq K\norm{\hat \theta - \mu}^2\right\},
\]
where $K$ is a universal constant from \cref{lem:basis}. Since $\mathcal{E}\subseteq \mathcal{G}$, we have
\begin{align}
\EE\left[ \frac{\left(\max\{\sum_{i \in \range{n}}c_i -\beta,0\}\right)^2}{\sum_{i \in \range{n}}c_i^2} \indic{\mathcal{G}}\right] & \geq  \EE\left[ \frac{\left(\max\{\sum_{i \in \range{n}}c_i -\beta,0\}\right)^2}{\sum_{i \in \range{n}}c_i^2} \indic{\mathcal{E}}\right]\nonumber\\
& \geq \EE\left[ \frac{\left(\max\{\sum_{i \in \range{n}}c_i -\beta,0\}\right)^2}{K\epsilon} \indic{\mathcal{E}}\right] \nonumber\\
&=  \EE\left[ \frac{\left(\max\{\sum_{i \in \range{n}}c_i -\beta,0\}\right)^2}{K\epsilon} \right] - \EE\left[ \frac{\left(\max\{\sum_{i \in \range{n}}c_i -\beta,0\}\right)^2}{K\epsilon} \indic{\mathcal{E}^c}\right] \label{eq:decomposition-expectation},
\end{align}
where the second step follows because under event $\mathcal{E}$, $\sum_{i \in \range{n}}c_i^2\leq K\norm{\hat \theta - \mu}^2$ and $\norm{\hat \theta - \mu}^2\leq \epsilon$. 
By convexity of $h_1(x)=x^2$ and $h_2(x)=\max\{x,0\}$, we can use Jensen's inequality to obtain
\[
\nonumber
\EE\left[ \frac{\left(\max\left\{\sum_{i \in \range{n}}c_i -\beta,0\right\}\right)^2}{K\epsilon} \right] &\geq \frac{\left(\max\left\{\EE\left[\sum_{i \in \range{n}}c_i\right] -\beta,0\right\}\right)^2}{K\epsilon}\\
&\geq \frac{\left(\frac{1}{3}-2\epsilon-3\delta - \beta\right)^2}{K\epsilon},
\]
where the last step follows from \cref{lem:fingerprinting-scvx}. Notice that by setting $\epsilon$ and $\delta$ sufficiently small, we have $\left(\frac{1}{3}-2\epsilon-3\delta - \beta\right)^2 \geq \Omega(1)$.

To upper bound the second term in \cref{eq:decomposition-expectation}, first, notice that the following holds with probability one
\[
\nonumber
\frac{\left(\max\{\sum_{i \in \range{n}}c_i -\beta,0\}\right)^2}{K\epsilon} &\leq \frac{\left(\sum_{i \in \range{n}}c_i\right)^2 +2\beta^2}{K\epsilon}\\
&\leq \frac{2\beta^2 + 16n^2}{K\epsilon},
\]
where the last step follows from 
\[
\nonumber
\norm{\sum_{i \in \range{n}}c_i} &= \norm{\inner{\hat\theta - \mu}{\sum_{i=1}^{n}\left(Z_{U_i,i}-\mu\right)}}\\
&\leq \norm{\hat\theta - \mu}\norm{\sum_{i=1}^{n}\left(Z_{U_i,i}-\mu\right)}\\
&\leq 4n.
\]
Then, in the next step, we provide an upper bound on $\Pr(\mathcal{E}^c)$. Union bound implies that
\[
\nonumber
\Pr\left(\mathcal{E}^c\right) \leq \Pr\left(\mathcal{G}^c\right) +  \Pr\left(\norm{\hat{\theta}-\mu}^2> \epsilon\right) + \Pr\left(\sum_{i \in \range{n}}c_i^2> K \norm{\hat \theta - \mu}^2\right).
\]
We want to set the parameters so that for a sufficiently small $\gamma$ the following hold 
\[
\label{eq:desired-probs}
\Pr\left(\mathcal{G}^c\right) \leq \gamma/n^2, \Pr\left(\norm{\hat{\theta}-\mu}^2> \epsilon\right) \leq \gamma/n^2,\Pr\left(\sum_{i \in \range{n}}c_i^2> K\norm{\hat \theta - \mu}^2\right)\leq \gamma/n^2.
\]
Notice that
\[
\nonumber
\Pr\left(\mathcal{G}^c\right)  &= \Pr\left(\max_{i \in \range{n}}\left\{\inner{\hat \theta - \mu}{Z_{\bar U_i,i}-\mu}\right\}\geq \beta/n\right)\\
&=\EE\left[\Pr\left(\max_{i \in \range{n}}\inner{\hat \theta - \mu}{Z_{\bar U_i,i}-\mu}\geq \beta/n \bigg| U,\hat \theta\right)\right].
\]
By the construction of CMI in \cref{def:cmi}, conditioned on $U$ and $\hat{\theta}$, $Z_{\bar U_i,i}$ is  \iid from $\Dist$ for $i \in \range{n}$. Therefore, we can use \cref{lem:indep-sample-corr}, to write
\[
\nonumber
\Pr\left(\mathcal{G}^c\right) &\leq n\EE\left[\Pr\left(\max_{i \in \range{n}}\inner{\hat \theta - \mu}{Z_{\bar U_i,i}-\mu}\geq \beta/n \bigg| U,\hat \theta\right)\right]\\
                            &\leq n \exp\left(-\frac{d}{8n^2}\right),
\]
We can see setting $d =  \Omega(n^2 \log(n^2))$, we have $\Pr\left(\mathcal{G}^c\right) \leq \gamma/n^2$ in \cref{eq:desired-probs}. Then, by the fact that $\Alg$ $\epsilon$-learns $\mathcal{P}^{(d)}_{\text{scvx}}$ and \cref{lem:scvx-learner-prop} we have
\[
\nonumber
\Pr\left(\norm{\hat{\theta}-\mu}^2 > \epsilon\right) \leq \delta = O\left(1/n^2\right).
\]
Also, by \cref{lem:basis}, given that $d= \Omega\left(n\right)$, we have
\[
\nonumber
\Pr\left(\sum_{i \in \range{n}}c_i^2 > K \norm{\hat \theta - \mu}\right) &= \Pr\left(\sum_{i \in \range{n}}\inner{\hat \theta - \mu}{Z_{U_i,i} - \mu}^2 > 6 \norm{\hat \theta - \mu}\right)\\
&= \EE\left[\Pr\left(\sum_{i \in \range{n}}\inner{\hat \theta - \mu}{Z_{U_i,i} - \mu}^2 > K \norm{\hat \theta - \mu}\bigg| U\right)\right]\\
&\leq O\left(1/n^2\right).
\]
In summary, we conclude that we can set the parameters such that in \cref{eq:decomposition-expectation}
\[
\nonumber
\EE\left[ \frac{\left(\max\{\sum_{i \in \range{n}}c_i -\beta,0\}\right)^2}{K\epsilon} \indic{\mathcal{E}^c}\right] \leq \frac{2\beta^2+16n^2}{K\epsilon} \Pr\left(\mathcal{G}^{c}\right) \leq O\left(\frac{1}{\epsilon}\right).
\]
Ergo, we conclude 
$
\EE[|\mathcal{I}|]\geq \EE\left[|\mathcal{I}|\indic{\mathcal{G}}\right] 
\geq \Omega\left(\frac{1}{\epsilon}\right).
$
and $\Pr\left(\mathcal{G}^c\right)\leq O(1/n)$. Therefore, from \cref{eq:cardinality-cmi-scvx}, we have
\[
\nonumber
\cmi &\geq \EE\left[|\mathcal{I}|\right] - 1 - n \Pr\left(\mathcal{G}^c\right) \\
&\geq  \Omega(1/\epsilon),
\]
for sufficiently small $\epsilon$.

\proofsubsection{thm:strong-cvx-upperbound}

The algorithm is based on subsampling a subset of training samples to create a new dataset and feeding it into an empirical risk minimizer.  
Let $0< m \leq n$ be constants to be determined later. Let the training set $\trainset_n = (Z_1,\dots,Z_n)$. The output of the algorithm $\hat{\theta}=\Alg_n(\trainset)$ is 
\[
\nonumber
\hat{\theta} = \argmin_{\theta \in \parspace} \Big\lbrace{\sum_{i \in \range{m}}f(\theta,Z_i) \Big\rbrace}.
\]
Notice that $\hat{\theta}$ is unique since $f(\cdot,z)$ is a strongly convex function. 
By \citep[Thm.~6]{shalev2009stochastic}, we have
\[
\nonumber
\EE[\Popriskcvx{\hat{\theta}}] - \min_{\theta \in \parspace} \Popriskcvx{\hat{\theta}}\leq \frac{4L^2}{\mu m}.
\]
Since $\Alg$ is a function of the first $m$ samples only, using the same argument as in the proof of \cref{thm:lip-bounded-upperbound}, we can show that $\cmi\leq m$. Finally, setting $m = \frac{4L^2}{\mu \epsilon}$ concludes the proof.

\subsection{Corollaries of Proof of \cref{thm:main-lower-stronglyconvex}}
\begin{corollary}\label{cor:num-cor-scvx}
  Let $\beta_{\text{scvx}},\epsilon_0,\delta_0$ be universal constants.  Let $\mathcal{P}^{(d)}_{\text{scvx}}$ be the problem instance described in \cref{sec:construction-scvx}. 
For every $\epsilon\leq\epsilon_0$ and $\delta\leq\delta_0$ and for every $\epsilon$-learner ($\Alg=\{\Alg_n\}_{n \in \Naturals}$), with sample complexity $N(\cdot,\cdot)$ the following holds: for every $n\geq N(\epsilon,\delta)$, $\delta<O(1/n^2)$, and $d= \Omega(n^2\log(n))$, there exists a data distribution $\Dist \in \ProbMeasures{\dataspace}$ such that
\[
\nonumber
\EE\left[\Big| \left\{i \in \range{n}: \inner{\hat \theta - \mu}{Z_i - \mu} \right\} \geq \frac{\beta_{\text{scvx}}}{ n} \Big|\right] = \Omega\left(\frac{1}{\epsilon}\right),
\]
where $\trainset_n = (Z_1,\dots,Z_n)\sim \Dist^{\otimes n}$, $\hat \theta = \Alg(\trainset_n)$, and $\mu = \EE_{Z\sim \Dist}\left[Z\right]$.
\end{corollary}

\begin{corollary}\label{cor:cardI-scvx}
Let $\beta_{\text{scvx}},\epsilon_0,\delta_0$ be universal constants. 
Consider the structure introduced in the definition of CMI in \cref{def:cmi}. Then, define the random variable
\[
\nonumber
\mathcal{I}=\left\{(i,j) \in \range{n}\times \{0,1\} \Big| \inner{\hat{\theta}-\mu}{Z_{j,i}-\mu}\geq \frac{\beta_{\text{scvx}}}{n}~\text{and}~\inner{\hat{\theta}-\mu}{Z_{1-j,i}-\mu}< \frac{\beta_{\text{scvx}}}{n}\right\},
\]
where $\trainset_n = (Z_{U_1,1},\dots,Z_{U_n,n})$, $\hat \theta = \Alg(\trainset_n)$, and $\mu = \EE_{Z\sim \Dist}\left[Z\right]$.
Let $\mathcal{P}^{(d)}_{\text{scvx}}$ be the problem instance described in \cref{sec:construction-scvx}. For every $\epsilon \leq \epsilon_0$ and $\delta \leq \delta_0$ and for every $\epsilon$-learner ($\Alg=\{\Alg_n\}_{n \in \Naturals}$), with sample complexity $N(\cdot,\cdot)$ the following holds: for every $n\geq N(\epsilon,\delta)$, $\delta<O(1/n^2)$, and $d= \Omega(n^2\log(n))$, there exists a data distribution $\Dist \in \ProbMeasures{\dataspace}$ such that
\[
\nonumber
\EE\left[\big| \mathcal{I}\big|\right]=\Omega\left(\frac{1}{\epsilon}\right).
\]

\end{corollary}
\section{Proof of Memorization Results}
\label{appx:memorizatin}
\newcommand{\empmu}{\mu_{\text{emp}}}
\newcommand{\heldout}{X}

\subsection{Adversary Strategy}
We describe the proposed strategy for the adversary in \cref{alg:attacker-convex} and \cref{alg:attacker-strongly-convex}.

\begin{algorithm}
\caption{ $\mathcal{Q}_{\text{cvx}}$: Adversary for Convex Losses}
\label{alg:attacker-convex}
\begin{algorithmic}[1]
\State Inputs: $\hat{\theta}\in \parspace$, $Z\in \dataspace$, $\Dist \in \ProbMeasures{\dataspace}$.
\State $\mu = \EE_{Z\sim \Dist}[Z]$
\State $ A = \mathrm{diag}\left[ \left\{\frac{144\epsilon^2 - d(\mu^{(k)})^2}{1-d(\mu^{(k)})^2}\right\}_{k=1}^{d}  \right] $
\State $\beta = \epsilon$.
\State $\mathcal{B}_{\text{FP}}=\varnothing$
\If{$\displaystyle \inner{\hat\theta }{A\left(Z - \mu\right)}\geq \frac{\beta}{n}$}
 \State $\hat{b}=1$
\Else
\State $\hat{b}=0$
\EndIf
\State Output $\hat{b}$
\end{algorithmic}
\end{algorithm}

\begin{algorithm}
\caption{$\mathcal{Q}_{\text{scvx}}$: Adversary for Strongly Convex Losses}
\label{alg:attacker-strongly-convex}
\begin{algorithmic}[1]
\State Inputs: $\hat{\theta}\in \parspace$, $Z\in \dataspace$, $\Dist \in \ProbMeasures{\dataspace}$.
\State $\mu = \EE_{Z\sim \Dist}[Z]$
\State $\beta = \beta_{\text{scvx}}$. \Comment{$\beta_{\text{scvx}}$ is from \cref{cor:cardI-scvx}.}
\If{$\displaystyle \inner{\hat\theta - \mu}{Z - \mu}\geq \frac{\beta}{4n}$}
 \State $\hat{b}=1$
\Else
\State $\hat{b}=0$
\EndIf
\State Output $\hat{b}$
\end{algorithmic}
\end{algorithm}

\begin{algorithm}
\caption{$\mathsf{FP}_{\text{scvx}}$: Fingerprint detector for Strongly Convex Losses}
\label{alg:fpl-strongly-convex}
\begin{algorithmic}[1]
\State Inputs: $\hat{\theta}\in \parspace$, $(Z_0,\dots,Z_n)\in \dataspace^{n+1}$, $\Dist \in \ProbMeasures{\dataspace}$.
\State $\mu = \EE_{Z\sim \Dist}[Z]$
\State $\beta = \beta_{\text{scvx}}$. \Comment{$\beta_{\text{scvx}}$ is from \cref{cor:cardI-scvx}.}
\State $\mathcal{B}_{\text{FP}}=\varnothing$
\For{$i \in \{0,\dots,n\}$}:
\If{$\displaystyle  \inner{\hat \theta - \mu}{Z_i - \mu}\geq\frac{\beta}{n}$}
\State $\mathcal{B}_{\text{FP}} = \mathcal{B}_{\text{FP}} \cup \{i\}$
\EndIf
\EndFor
\State Output $\mathcal{B}_{\text{FP}}$
\end{algorithmic}
\end{algorithm}

\begin{algorithm}
\caption{$\mathsf{CR}_{\text{scvx}}$: Correlation-Reduction for Strongly Convex Losses}
\label{alg:corrred-strongly-convex}
\begin{algorithmic}[1]
\State Inputs: $\hat{\theta}\in \parspace$, $(Z_1,\dots,Z_n)\in \dataspace^n$, $Z_0 \sim \Dist$
\State $\tilde \mu = Z_0$
\State $\beta = \beta_{\text{scvx}}$. \Comment{$\beta_{\text{scvx}}$ is from \cref{cor:cardI-scvx}.}
\State $\mathcal{B}_{\text{corr-red}}=\varnothing$
\State $w = \hat \theta$
\For{$i \in \range{n}$}:
\If{$ \displaystyle  \inner{\hat \theta}{Z_i - \tilde{\mu}}\geq\frac{\beta}{2n}$}
\State $\mathcal{B}_{\text{corr-red}} = \mathcal{B}_{\text{corr-red}} \cup \{i\}$
    \If{$\big|\mathcal{B}_{\text{corr-red}}\big|=\frac{2}{\epsilon}\log\left(\frac{1}{\delta}\right)$}
        \State Sample $\mathcal{R}\subseteq \range{n}$ a uniform random subset of size $\frac{2}{\epsilon}\log(\frac{1}{\delta})$ from $\range{n}$
        \State $w = \empmu\left(\mathcal{R}\right)$ \Comment{$\empmu\left(\mathcal{R}\right)$ denotes the empirical mean of the data points with the index in $\mathcal{R}$.}
        \State \textbf{Break}
    \EndIf
\EndIf
\EndFor
\State Output $w, \mathcal{B}_{\text{corr-red}}$
\end{algorithmic}
\end{algorithm}

\proofsubsection{thm:membership-cvx}
Let $b = (b_1,\dots,b_n)$ denote the outcome of fair coin at each round of the game described in \cref{def:mem-game}.  Then, let $\hat b_i = \mathcal{Q}_{\text{cvx}}\left(\hat \theta, Z_{b_i,i},\mathcal{D}\right)$ for each round $i\in \range{n}$ and let us denote the output of the adversary as $(\hat b_1,\dots, \hat b_n)\in \{0,1\}^{ n}$ where $\mathcal{Q}_{\text{cvx}}$ is given by \cref{alg:attacker-convex}.

\subsubsection{Soundness Analysis}
Define the following event 
\[
\nonumber
\mathcal{G} = \left\{\forall i \in \range{n}:  \inner{\hat{\theta} }{A\left(Z_{0,i}-\mu\right)}<\beta/n\right\}.
\]

Notice that
\[
\nonumber
&\Pr\left( \exists i \in \range{n}\!:\! \mathcal{Q}_{\text{cvx}}\left(\hat{\theta},Z_{0,i},\Dist\right)=1 \right) \\
&= \Pr\left( \right\{\exists i \in \range{n}\!:\! \mathcal{Q}_{\text{cvx}}\left(\hat{\theta},Z_{0,i},\Dist\right)=1\left\}~\wedge~\mathcal{G} \right)  +  \Pr\left( \right\{\exists i \in \range{n}\!:\! \mathcal{Q}_{\text{cvx}}\left(\hat{\theta},Z_{0,i},\Dist\right)=1\left\}~\wedge~\mathcal{G}^c \right)\\
&\leq \Pr\left( \right\{\exists i \in \range{n}\!:\! \mathcal{Q}_{\text{cvx}}\left(\hat{\theta},Z_{0,i},\Dist\right)=1 \left\}~\wedge~\mathcal{G} \right) + \Pr\left( \mathcal{G}^c \right).
\]
We claim that $\Pr\left( \right\{\exists i \in \range{n}\!:\! \mathcal{Q}_{\text{cvx}}\left(\hat{\theta},Z_{0,i},\Dist\right)=1 \left\}~\wedge~\mathcal{G} \right) = 0$. It follows from the following observation: $\exists i \in \range{n}\!:\! \mathcal{Q}_{\text{cvx}}\left(\hat{\theta},Z_{0,i},\Dist\right)=1$ can happen if and only if there exists $i \in \range{n}$ such that $ \inner{\hat{\theta} }{A\left(Z_{0,i}-\mu\right)}\geq \beta/n$. However, the intersection of this event with $\mathcal{G}$ is empty by the definition of $\mathcal{G}$. Therefore, we can write
\[
\nonumber
\Pr\left(\exists i \in \range{n}\!:\! \mathcal{Q}_{\text{cvx}}\left(\hat{\theta},Z_{0,i},\Dist\right)=1 \right) \leq \Pr\left( \mathcal{G}^c \right).
\]

To upper bound $\Pr\left( \mathcal{G}^c \right)$, notice
\[
\nonumber
\Pr\left(\mathcal{G}^c\right) &= \Pr\left( \exists i \in \range{n}: \inner{\hat{\theta} }{A\left(Z_{0,i}-\mu\right)}\geq \beta/n\right).
\]
By the facts that $\hat \theta \indep Z_{0,i}$ for every $i \in \range{n}$, $A$ is a diagonal matrix, $\norm{A}_2 \leq 144\epsilon^2$, and $\norm{\hat{\theta}}\leq 1$, we can use \cref{lem:indep-sample-corr} to write
\[
\nonumber
\Pr\left( \exists i \in \range{n}: \inner{A\hat{\theta} }{Z_{0,i}-\mu}\geq \beta/n\right)\leq n \exp\left( - \frac{\epsilon^2 d}{2\cdot n^2 \cdot (144 \epsilon^2)^2 } \right) \leq \xi,
\]
given $d \geq \Omega(n^2 \log(n/\xi))$. Notice that by assumption $\epsilon<1$. This concludes the soundness analysis.

\subsubsection{Recall Analysis}
The construction of the hard problem instance is given in \cref{sec:construction-cvx}. Let $\Alg$ be an arbitrary $\epsilon$-learner and let $\Dist$ be a distribution that satisfies \cref{lem:fingerprinting-cvx}. Define the following set
\[
\mathcal{I} = \{i \in \range{n}:  \inner{\hat\theta }{A\left(Z_{1,i} - \mu\right)}\geq \frac{\beta}{n} \}
\]
This set includes the subset of training samples that the adversary could identify. In \cref{cor:num-cor-cvx}, we showed that $\EE\left[\mathcal{I}\right]=\Omega\left(\frac{1}{\epsilon^2}\right)$. Moreover, by the assumption on $n$, we have $n= \Theta\left(\frac{\log(1/\delta)}{\epsilon^2}\right)$. Notice that $\mathcal{I}\leq n$ with probability one. We invoke \cref{lem:reverse-markov} to write
\[
\Pr\left(\mathcal{I}=\Omega\left(\frac{1}{\epsilon^2}\right)\right) \geq p_0 \left(\log(1/\delta)\right)^{-1},
\]
as was to be shown where $p_0$ is a universal constant.

\proofsubsection{thm:membership-scvx}

Let $b = (b_1,\dots,b_n)$ denote the outcome of a fair coin at each round of the game described in \cref{def:mem-game}.  Then, let $\hat b_i = \mathcal{Q}_{\text{scvx}}\left(\hat \theta, Z_{b_i,i},\mathcal{D}\right)$ for each round $i\in \range{n}$ and let us denote the output of the adversary as $(\hat b_1,\dots, \hat b_n)\in \{0,1\}^{ n}$.

\subsubsection{Soundness Analysis}
Define the following event 
\[
\nonumber
\mathcal{G} = \left\{\forall i \in \range{n}: \inner{\hat{\theta} - \mu}{Z_{0,i}-\mu}<\beta/n\right\}.
\]
Notice that
\[
\nonumber
&\Pr\left( \exists i \in \range{n}\!:\! \mathcal{Q}_{\text{scvx}}\left(\hat{\theta},Z_{0,i},\Dist\right)=1\right) \\
&= \Pr\left( \right\{\exists i \in \range{n}\!:\! \mathcal{Q}_{\text{scvx}}\left(\hat{\theta},Z_{0,i},\Dist\right)=1\left\}~\wedge~\mathcal{G} \right)  +  \Pr\left( \right\{\exists i \in \range{n}\!:\! \mathcal{Q}_{\text{scvx}}\left(\hat{\theta},Z_{0,i},\Dist\right)=1\left\}~\wedge~\mathcal{G}^c \right)\\
&\leq \Pr\left( \right\{\exists i \in \range{n}\!:\! \mathcal{Q}_{\text{scvx}}\left(\hat{\theta},Z_{0,i},\Dist\right)=1 \left\}~\wedge~\mathcal{G} \right) + \Pr\left( \mathcal{G}^c \right).
\]
We claim that $\Pr\left( \right\{\exists i \in \range{n}\!:\! \mathcal{Q}_{\text{scvx}}\left(\hat{\theta},Z_{0,i},\Dist\right)=1 \left\}~\wedge~\mathcal{G} \right) = 0$. It follows from the following observation: $\mathcal{Q}_{\text{scvx}}\left(\hat{\theta},Z_{0,i},\Dist\right)=1$ can happen if and only if $\inner{\hat \theta - \mu}{Z_{0,i}-\mu}\geq \beta/n$. However, the intersection of this event with $\mathcal{G}$ is empty by the definition of $\mathcal{G}$. Therefore, we can write
\[
\nonumber
\Pr\left( \exists i \in \range{n}\!:\! \mathcal{Q}_{\text{scvx}}\left(\hat{\theta},Z_{0,i},\Dist\right)=1 \right) \leq \Pr\left( \mathcal{G}^c \right).
\]
Since $Z_{0,i}\indep \hat \theta$, we can use \cref{lem:indep-sample-corr} to write
\[
\nonumber
\Pr\left( \mathcal{G}^c \right) \leq n \exp\left( - \frac{d}{4 n^2} \right).
\]
By setting $d\geq \Omega(n^2 \log(n/\xi))$, we obtain that 
\[
\nonumber
\Pr\left( \exists i \in \range{n}\!:\!\mathcal{Q}_{\text{scvx}}\left(\hat{\theta},Z_{0,i},\Dist\right)=1 \right) \leq \xi.
\]
This concludes the soundness analysis.

\subsubsection{Recall Analysis}

The construction of the hard problem instance is given in \cref{sec:construction-scvx}. Let $\Alg$ be an arbitrary $\epsilon$-learner. The data distribution $\Dist$ is a product distribution over $\{\pm 1/\sqrt{d}\}^d$ and will be determined later. Consider the algorithms given in \cref{alg:attacker-strongly-convex,alg:fpl-strongly-convex,alg:corrred-strongly-convex} and using them define the following random variables:
\[
\nonumber
(Z_0,Z_1,\dots,Z_n) &\sim \Dist^{\otimes (n+1)},\\
\hat \theta & = \Alg_n(Z_1,\dots,Z_n),\\
\mathcal{B}_{\text{adversary}}&=\{i \in \range{n}: \hat b_i=1 \},\\
w,\mathcal{B}_{\text{corr-red}}&=\mathsf{CR}_{\text{scvx}}\left(\hat \theta,(Z_1,\dots,Z_n),Z_0\right),\\
\mathcal{B}_{\text{FP}} &= \mathsf{FP}_{\text{scvx}}\left(w,(Z_0,\dots,Z_n),\Dist\right).
\]
In particular, $Z_0$ is a sample drawn from $\Dist$ which is independent of the training set, i.e., $(Z_1,\dots,Z_n)$. 

Recall that our goal is to show that $\Pr\left(\big|\mathcal{B}_{\text{adversary}}\big|=\Omega(1/\epsilon)\right)$ is greater than a universal constant. Our approach is as follows:
In the first step, we show that, with a high probability, $\mathcal{B}_{\text{corr-red}} \subseteq \mathcal{B}_{\text{adversary}}$. Then, in the second step, we will show that with a high probability $|\mathcal{B}_{\text{FP}}| \leq |\mathcal{B}_{\text{corr-red}}| + 1$. In the third step, we will show that $\EE\left[\big| \mathcal{B}_{\text{FP}}\big|\right]=\Omega(1/\epsilon)$ which gives us  $\EE\left[\big| \mathcal{B}_{\text{corr-red}}\big|\right]=\Omega(1/\epsilon)$. Finally, we argue that $\mathcal{B}_{\text{corr-red}} = \tilde{O}\left(1/\epsilon\right)$ with probability one, and using reverse Markov's inequality, we show that $\Pr(\big|\mathcal{B}_{\text{corr-red}}\big| \geq \Omega(1/\epsilon))$ is greater than a universal constant. Combining this result with Step 1 concludes the proof.

\paragraph{Step 1: with a high probability, $\mathcal{B}_{\text{corr-red}} \subseteq \mathcal{B}_{\text{adversary}}$.} 
Simple calculations show that 
\[
\label{eq:connection-scvx-corr-red}
\inner{\hat \theta}{Z_i - \tilde{\mu}} = \inner{\hat \theta - \mu}{Z_i - \mu} + \inner{\hat \theta}{\mu - \tilde \mu} + \inner{\mu}{Z_i - \mu}.
\]
Then, we can write
\[
\nonumber
\Pr\left(\mathcal{B}_{\text{corr-red}} \not\subseteq \mathcal{B}_{\text{adversary}}\right)&=\Pr\left(\exists i \in \range{n}: \inner{\hat \theta}{Z_i - \tilde{\mu}}  \geq \frac{\beta}{2n} \wedge \inner{\hat \theta - \mu}{Z_i - \mu}  < \frac{\beta}{4n} \right) \\
&\leq \Pr \left( \exists i \in \range{n}:\inner{\hat \theta}{\mu - \tilde \mu}  +  \inner{\mu}{Z_i - \mu}\geq \frac{\beta}{4n}  \right),
\]
where the last step follows from \cref{eq:connection-scvx-corr-red}. Then, 
\[
\nonumber
\Pr\left(\mathcal{B}_{\text{corr-red}} \not\subseteq \mathcal{B}_{\text{adversary}}\right)&\leq  \Pr \left(\inner{\hat \theta}{\mu - \tilde \mu} \geq \frac{\beta}{4n}  \right) + \Pr\left( \exists i \in \range{n}:  \inner{\mu}{Z_i - \mu} \geq \frac{\beta}{4n} \right)\\
&\leq \exp\left(-\frac{d\cdot\beta^2}{32 n^2 \norm{\mu}^2}\right) + n \exp\left(-\frac{d\beta^2}{32 n^2 \norm{\mu}^2}\right)\\
&\leq \left(n+1\right) \exp\left(-\frac{d\cdot\beta^2}{32n^2}\right),
\]
where the first step follows from union bound and the second step follows from \cref{lem:indep-sample-corr}. This shows that setting $d = \Omega\left(n^2 \log(n^2)\right)$, we obtain 
\[
\nonumber
\Pr\left(\exists i \in \range{n}: \inner{\hat \theta}{Z_i - \tilde{\mu}}  \geq \frac{\beta}{2n} \wedge \inner{\hat \theta - \mu}{Z_i - \mu}  < \frac{\beta}{4n} \right) \leq O\left(\frac{1}{n}\right).
\]
This is equivalent to 
\[
\nonumber
\Pr\left(\mathcal{B}_{\text{corr-red}} \subseteq \mathcal{B}_{\text{adversary}}\right) \geq 1 - O\left(\frac{1}{n}\right).
\]
\paragraph{Step 2: with a high probability,  $|\mathcal{B}_{\text{FP}}| \leq |\mathcal{B}_{\text{corr-red}}| + 1$.}

Notice that $\big|\mathcal{B}_{\text{FP}}\big| = \big|\mathcal{B}_{\text{FP}} \cap \{1,\dots,n\}\big| + \big|\mathcal{B}_{\text{FP}} \cap \{0\}\big|$. 
We can write 
\[
\label{eq:bcorr-bfp}
\Pr\left(\big|\mathcal{B}_{\text{FP}} \cap \{1,\dots,n\}\big| > \big|\mathcal{B}_{\text{corr-red}}\big| \right) 
&= \Pr\left(\left\{\big|\mathcal{B}_{\text{FP}} \cap \{1,\dots,n\}\big| > \big|\mathcal{B}_{\text{corr-red}}\big| \right\}\wedge \{w = \hat \theta\}\right) \\
&+ \Pr\left(\left\{\big|\mathcal{B}_{\text{FP}} \cap \{1,\dots,n\}\big| > \big|\mathcal{B}_{\text{corr-red}}\big| \right\} \wedge \{w = \theta_0\}\right),
\]
where $\theta_0$ denotes the output in the case of outputting the empirical mean.
For the first term in \cref{eq:bcorr-bfp}, we can write
\[
\nonumber
&\Pr\left(\left\{\big|\mathcal{B}_{\text{FP}} \cap \{1,\dots,n\}\big| > \big|\mathcal{B}_{\text{corr-red}}\big| \right\}\wedge \{w = \hat \theta\}\right)\\
&\leq \Pr\left(\left\{\exists i \in \{1,\dots,n\}: i \in \mathcal{B}_{\text{FP}} \wedge i \notin \mathcal{B}_{\text{corr-red}}\right\}\wedge \{w = \hat \theta\}\right) \\
&= \Pr\left(\left\{\exists i \in \{1,\dots,n\}: \inner{\hat \theta - \mu}{Z_i - \mu}\geq \frac{\beta}{n} \wedge \inner{\hat \theta }{Z_i - \tilde{\mu}} < \frac{\beta}{2n} \right\}\wedge \{w = \hat \theta\}\right) \\
& \leq \Pr\left(\left\{\exists i \in \{1,\dots,n\}: \inner{\hat \theta - \mu}{Z_i - \mu}\geq \frac{\beta}{n} \wedge \inner{\hat \theta }{Z_i - \tilde{\mu}} < \frac{\beta}{2n} \right\}\right).
\]
Notice that we have
\[
\nonumber
\inner{\hat \theta - \mu}{Z_i - \mu} = \inner{\hat \theta}{Z_i - \tilde \mu} + \inner{\hat \theta}{\tilde \mu - \mu} - \inner{\mu}{Z_i - \mu}.
\]
Using this equality, we can write
\[
\nonumber
&\Pr\left(\left\{\exists i \in \{1,\dots,n\}: \inner{\hat \theta - \mu}{Z_i - \mu}\geq \frac{\beta}{n} \wedge \inner{\hat \theta }{Z_i - \tilde{\mu}} < \frac{\beta}{2n} \right\}\right) \\
&\leq \Pr \left(\inner{\hat \theta}{\mu - \tilde \mu} \geq \frac{\beta}{4n}  \right) + \Pr\left( \exists i \in \range{n}:  \inner{\mu}{ \mu - Z_i} \geq \frac{\beta}{4n} \right)\\
&\leq \exp\left(-\frac{d\cdot\beta^2}{32 n^2 \norm{\mu}^2}\right) + n \exp\left(-\frac{d\beta^2}{32 n^2 \norm{\mu}^2}\right)\\
&\leq \left(n+1\right) \exp\left(-\frac{d\cdot\beta^2}{32n^2}\right),
\]
where the first step follows from the union bound and the  step follows from \cref{lem:indep-sample-corr} since $\tilde \mu \indep \hat \theta$. Therefore, setting $d = \Omega(n^2 \log(n))$, we obtain that this term is at most $O(1/n)$.

Then, for the second term in \cref{eq:bcorr-bfp},
\[
\nonumber
&\Pr\left(\{\big| \mathcal{B}_{\text{FP}}\cap\{1,\dots,n\} \big|> \big| \mathcal{B}_{\text{corr-red}}   \big|\} \wedge \{w=\empmu\left(\mathcal{R}\right)\}\right)\\
&= \Pr\left(\{\big| \mathcal{B}_{\text{FP}}\cap\{1,\dots,n\} \big|> \frac{2}{\epsilon}\log(1/\delta)  \} \wedge \{w=\empmu\left(\mathcal{R}\right)\}\right),
\]
where the last line follows because under the event $w=\empmu\left(\mathcal{R}\right)$, $\big| \mathcal{B}_{\text{corr-red}}   \big| = \frac{2}{\epsilon}\log\left(1/\delta\right)$ by the description of \cref{alg:corrred-strongly-convex}. Notice that $|\mathcal{R}|=\frac{2}{\epsilon}\log(1/\delta)$ and $\mathcal{R}$ is independent of every other random variables. Therefore, the event $\big|\mathcal{B}_{\text{FP}}\cap \{1,\dots,n\}\big|>\frac{2}{\epsilon}\log(1/\delta)$ is a subset of the event that there exists $i \notin \mathcal{R}$ such that $\inner{w-\mu}{Z_i - \mu}>\frac{\beta}{n}$. However, notice that $\empmu(\mathcal{R}) \indep Z_i$ by the description of \cref{alg:corrred-strongly-convex} for $i \neq R$. Therefore, we can write
\[
\nonumber
&\Pr\left(\{\big| \mathcal{B}_{\text{FP}}\cap\{1,\dots,n\} \big|> \frac{2}{\epsilon}\log(1/\delta)  \} \wedge \{w=\empmu\left(\mathcal{R}\right)\}\right)  \\
&\leq \EE\left[\Pr\left( \exists i \notin \mathcal{R}: \inner{\empmu(\mathcal{R}) - \mu}{Z_i -\mu}\geq \frac{\beta}{n} \wedge \{w=\empmu\left(\mathcal{R}\right)\}\bigg| \mathcal{R}\right)\right]
\\
&\leq \EE\left[\Pr\left( \exists i \notin \mathcal{R}: \inner{\empmu(\mathcal{R}) - \mu}{Z_i -\mu}\geq \frac{\beta}{n} \bigg| \mathcal{R}\right)\right].
\]
By an application of \cref{lem:indep-sample-corr}, we have 
\[
\nonumber
\Pr\left( \exists i \notin \mathcal{R}: \inner{\empmu(\mathcal{R}) - \mu}{Z_i -\mu}\geq \frac{\beta}{n} \bigg| \mathcal{R}\right) \leq n\cdot \exp\left(-\frac{d \beta^2}{32 n^2}\right).
\]
It can be seen by setting $d = \Omega(n^2\log(n^2))$, we obtain that this probability is at most $O(1/n)$. Therefore, combining these two upper bounds with \cref{eq:bcorr-bfp} shows that with probability at least $1-O(1/n)$, we have 
\[
\nonumber
\big|\mathcal{B}_{\text{FP}}\big| &= \big|\mathcal{B}_{\text{FP}} \cap \{1,\dots,n\}\big| + \big|\mathcal{B}_{\text{FP}} \cap \{0\}\big|\\
&\leq \big| \mathcal{B}_{\text{corr-red}} \big| + 1,
\]
as was to be shown.
\paragraph{Step 3: $\EE\left[\mathcal{B}_{\text{FP}}\right]=\Omega(1/\epsilon)$.}
In the first step, we claim that $w$ (output of \cref{alg:corrred-strongly-convex}) satisfies the definition of $\epsilon$-learner. The reason is as follows: $w$ can be either $\hat \theta = \Alg_n(\trainset_n)$ or $\empmu\left(\mathcal{R}\right)$. Notice that $\Alg_n$ is an $\epsilon$-learner by assumption. Consider the case that $w = \empmu \left(\mathcal{R}\right)$. Then, 
\[
\nonumber
\Pr\left(\norm{\empmu\left(\mathcal{R}\right) - \mu}^2 > \epsilon\right) &= \EE\left[ \Pr\left(\norm{\empmu\left(\mathcal{R}\right) - \mu}^2 > \epsilon \big| \mathcal{R}\right)\right]\\
&\leq \delta,
\]
where the last step follows from \cref{lem:subgaussian-randomvector}. Also, by the description of the problem instance in \cref{sec:construction-scvx} we have for every $\Dist$ and $\theta$, 
\[
\nonumber
\Popriskcvx{\theta} -  \Popriskcvx{\theta^\star} &= \frac{1}{2}\norm{\theta - \mu}^2.
\]

Therefore, by union bound we see that the output of \cref{alg:corrred-strongly-convex} has an excess error of $\epsilon$, with probability at least $1-2\delta$ with the sample complexity of $N(\epsilon,\delta)$ where $N$ is the sample complexity of $\Alg$. 

In \cref{cor:num-cor-scvx} we showed that for every $\epsilon$-learner, we can find a data distribution $\Dist$ such that 
\[
\nonumber
\EE\left[\big|\mathcal{B}_{\text{FP}}\big|\right]=\Omega\left(\frac{1}{\epsilon}\right).
\]
In particular, we choose $\Dist$ to achieve this lowerbound for $w$  (output of \cref{alg:corrred-strongly-convex}).
\paragraph{Step 4: Conclusion.}
First, we provide a lower bound on the $\EE[|\mathcal{B}_{\text{corr-red}}|]$ as follows
\[
\nonumber
\EE\left[\big| \mathcal{B}_{\text{corr-red}}\big|\right] &= \EE\left[\big| \mathcal{B}_{\text{corr-red}}\big|\cdot \indic{|\mathcal{B}_{\text{corr-red}}| +1 \geq |\mathcal{B}_{\text{FP}}| }\right] + \EE\left[\big| \mathcal{B}_{\text{corr-red}}\big|\cdot \indic{|\mathcal{B}_{\text{corr-red}}| +1 < |\mathcal{B}_{\text{FP}}| }\right]\\
&\geq   \EE\left[\big| \mathcal{B}_{\text{FP}}\big|\cdot \indic{|\mathcal{B}_{\text{corr-red}}| +1 \geq |\mathcal{B}_{\text{FP}}| }\right] -1 + \EE\left[\big| \mathcal{B}_{\text{corr-red}}\big|\cdot \indic{|\mathcal{B}_{\text{corr-red}}| +1 < |\mathcal{B}_{\text{FP}}| }\right]\\
&\geq \EE\left[\big| \mathcal{B}_{\text{FP}}\big|\right] -1  + \EE\left[\left(\big| \mathcal{B}_{\text{corr-red}}\big| - \big| \mathcal{B}_{\text{FL}}\big| \right)\cdot \indic{|\mathcal{B}_{\text{corr-red}}| +1 < |\mathcal{B}_{\text{FP}}| }\right]\\
&\geq \EE\left[\big| \mathcal{B}_{\text{FP}}\big|\right] -1 -n \Pr\left(|\mathcal{B}_{\text{corr-red}}| +1 < |\mathcal{B}_{\text{FP}}|\right)\\
&\geq \Omega\left(\frac{1}{\epsilon}\right),
\]
where the last step follows from Step 2 and Step 3 for sufficiently small $\epsilon$.  By the description of the random variable $\big| \mathcal{B}_{\text{corr-red}}\big|$ in \cref{alg:corrred-strongly-convex}, with probability one $\big| \mathcal{B}_{\text{corr-red}}\big| \leq \frac{2}{\epsilon}\log\left(\frac{1}{\delta}\right)$. 
Then, we invoke reverse Markov's inequality from \cref{lem:reverse-markov} gives
\[
\nonumber
\Pr\left(\big| \mathcal{B}_{\text{corr-red}}\big| = \Omega\left(\frac{1}{\epsilon}\right)\right)\geq p_0 \left(\log(1/\delta)\right)^{-1},
\]
where $p_0$ is a universal constant. Also, we showed in Step 1 that
\[
\nonumber
\Pr\left(\big| \mathcal{B}_{\text{adversary}}\big| \geq \big| \mathcal{B}_{\text{corr-red}}\big| \right) \geq 1 - O\left(\frac{1}{n}\right).
\]
Combining these two facts using union bound gives us
\[
\nonumber
\Pr\left(\big| \mathcal{B}_{\text{adversary}}\big| = \Omega\left(\frac{1}{\epsilon}\right)\right)\geq p_0 \left(\log(1/\delta)\right)^{-1} + O(1/n),
\]
as was to be shown. Note that $n$ is at least $1/\epsilon$, therefore, for a sufficiently small $\epsilon$, we have the desired result.

\section{Proofs of Lower Bound for Individual-Sample CMI}
\label{appx:indiv-sample}

In this part, we show that our proof techniques for \cref{thm:main-lower-convex} and \cref{thm:main-lower-stronglyconvex} easily extend to ISCMI.  First, we begin with the strong convex case. Let $\beta=\beta_{\text{scvx}}$ as in \cref{cor:cardI-scvx} and define  
\[
\nonumber
\mathcal{I}=\left\{(i,j) \in \range{n}\times \{0,1\} \Big| \inner{\hat{\theta}-\mu}{Z_{j,i}-\mu}\geq \beta/n~\text{and}~\inner{\hat{\theta}-\mu}{Z_{1-j,i}-\mu}< \beta/n\right\}.
\]
Also, define $\mathcal{I}^{(1)}$ as follows: $i \in \mathcal{I}^{(1)}$ iff $\exists j \in \{0,1\}$ such that $(i,j)\in \mathcal{I}$. In words, $\mathcal{I}^{(1)}$ represents the set of coulmns for which there is a significant gap between the correlations.

We also introduce the following events
\[
\nonumber
\mathcal{G}_i = \left\{\inner{\hat{\theta}-\mu}{Z_{\bar{U_i},i}-\mu}<\beta/n\right\}, \quad
\mathcal{M}_i = \left\{i \in \mathcal{I}^{(1)}\right\},
\]
where $\bar U_i = 1 - U_i$. 

We can simplify the mutual information term in ISCMI as follows
\[
\nonumber
\sum_{i=1}^{n}\minf{\hat{\theta},U_i \big| Z_{0,i},Z_{1,i}} &= n - \sum_{i=1}^{n} \entr{U_i \big| \hat{\theta},Z_{0,i},Z_{1,i}},
\]
where the last step follows from $U_i \indep (Z_{0,i},Z_{1,i})$.

In the next step, for every $i \in \range{n}$, we provide an upper bound on $ \entr{U_i \big| \hat{\theta},Z_{0,i},Z_{1,i}}$. First, notice that $\indic{\mathcal{M}_i}$ is a $\left(\hat \theta, Z_{0,i},Z_{1,i}\right)$-measurable random variable. Therefore,
\[
\label{eq:expansion-iscmi}
\entr{U_i \big| \hat{\theta},Z_{0,i},Z_{1,i}} = \entr{U_i \big| \hat{\theta},Z_{0,i},Z_{1,i},\indic{\mathcal{M}_i}}.
\]
Using the monotonicity and chain rule of entropy, we can write
\[
\nonumber
\entr{U_i \big| \hat{\theta},Z_{0,i},Z_{1,i},\indic{\mathcal{M}_i}} &\leq \entr{U_i ,\indic{\mathcal{G}_i} \big| \hat{\theta},Z_{0,i},Z_{1,i},\indic{\mathcal{M}_i}}\\
&= \entr{\indic{\mathcal{G}_i} \big| \hat{\theta},Z_{0,i},Z_{1,i},\indic{\mathcal{M}_i}} + \entr{U_i  \big| \hat{\theta},Z_{0,i},Z_{1,i},\indic{\mathcal{M}_i},\indic{\mathcal{G}_i}}\\
&\leq \entr{\indic{\mathcal{G}_i}} + \entr{U_i  \big| \hat{\theta},Z_{0,i},Z_{1,i},\indic{\mathcal{M}_i},\indic{\mathcal{G}_i}} \\
& = \mathrm{H}_b\left(\Pr\left(\mathcal{G}_i^c\right)\right) + \entr{U_i  \big| \hat{\theta},Z_{0,i},Z_{1,i},\indic{\mathcal{M}_i},\indic{\mathcal{G}_i}},
\]
where the third step follows because conditioning does not increase entropy and the last step follows because $\indic{\mathcal{G}_i}$ is a binary random variable. Then, we can write 
\[
\nonumber
\entr{U_i  \big| \hat{\theta},Z_{0,i},Z_{1,i},\indic{\mathcal{M}_i},\indic{\mathcal{G}_i}} &= \entr{U_i  \big| \hat{\theta},Z_{0,i},Z_{1,i},\indic{\mathcal{G}_i},\indic{\mathcal{M}_i}=0} \Pr\left(\indic{\mathcal{M}_i}=0\right)\\
&+ \entr{U_i  \big| \hat{\theta},Z_{0,i},Z_{1,i},\indic{\mathcal{G}_i}=1,\indic{\mathcal{M}_i}=1} \Pr\left(\indic{\mathcal{M}_i}=1 \wedge \indic{\mathcal{G}_i}=1\right) \\
&+\entr{U_i  \big| \hat{\theta},Z_{0,i},Z_{1,i},\indic{\mathcal{G}_i}=0,\indic{\mathcal{M}_i}=1} \Pr\left(\indic{\mathcal{M}_i}=1 \wedge \indic{\mathcal{G}_i}=0\right).
\]
We use the following estimates for each term. Since $U_i$ is a binary random variable, we have
\[
\nonumber
\entr{U_i  \big| \hat{\theta},Z_{0,i},Z_{1,i},\indic{\mathcal{G}_i},\indic{\mathcal{M}_i}=0} \Pr\left(\indic{\mathcal{M}_i}=0\right) \leq \Pr\left(\indic{\mathcal{M}_i}=0\right).
\]
Then, for the second term, conditioned on  $\indic{\mathcal{G}_i}=\indic{\mathcal{M}_i}=1$, $U_i$ is given by $j$ where $(i,j)\in \mathcal{I}$ since
\[
\nonumber
\{(i,j)\in \mathcal{I}\} \cap \mathcal{G}_i &\Rightarrow  \left\{ \inner{\hat{\theta}-\mu}{Z_{j,i}-\mu} \geq \beta/n~\text{and}~\inner{\hat{\theta}-\mu}{Z_{1-j,i}-\mu} < \beta/n \right\}  \cap \left\{ \inner{\hat{\theta}-\mu}{Z_{\bar{U}_i,i}-\mu}< \beta/n \right\}\\
&\Rightarrow \{j=U_i\}.
\]

Therefore,
$$
\entr{U_i  \big| \hat{\theta},Z_{0,i},Z_{1,i},\indic{\mathcal{G}_i}=1,\indic{\mathcal{M}_i}=0}=0.
$$
For the third term, since $U_i$ is a binary random variable, we can write
\[
\nonumber
\entr{U_i  \big| \hat{\theta},Z_{0,i},Z_{1,i},\indic{\mathcal{G}_i}=0,\indic{\mathcal{M}_i}=0}\Pr\left(\indic{\mathcal{M}_i}=1 \wedge \indic{\mathcal{G}_i}=0\right) \leq \Pr\left(\indic{\mathcal{G}_i}=0\right).
\]
In summary, we showed that
\[
\nonumber
\entr{U_i \big| \hat{\theta},Z_{0,i},Z_{1,i}}  \leq \Pr\left(\indic{\mathcal{G}_i}=0\right) + \mathrm{H}_b\left(\Pr\left(\mathcal{G}_i^c\right)\right)  +\Pr\left(\indic{\mathcal{M}_i}=0\right).
\]
Using it, we can upper bound the sum of the conditional entropy as 
\[
\label{eq:iscmi-simpified}
 \sum_{i=1}^{n} \entr{U_i \big| \hat{\theta},Z_{0,i},Z_{1,i}} &\leq \sum_{i=1}^{n}\mathrm{H}_b\left(\Pr\left(\mathcal{G}_i^c\right)\right) + \Pr\left(\mathcal{G}_i^c\right) + \Pr\left(\indic{\mathcal{M}_i}=0\right)\\
 & = \sum_{i=1}^{n} \EE\left[\indic{\indic{\mathcal{M}_i}=0}\right] + \Pr\left(\mathcal{G}_i^c\right) + \mathrm{H}_b\left(\Pr\left(\mathcal{G}_i^c\right)\right)\\
  & = \EE\left[(n-\big| \mathcal{I} \big|)\right]+ \sum_{i=1}^{n}\Pr\left(\mathcal{G}_i^c\right) + \sum_{i=1}^{n}\mathrm{H}_b\left(\Pr\left(\mathcal{G}_i^c\right)\right),
\]
where the last term follows because $\sum_{i=1}^{n}\EE\left[\indic{\indic{\mathcal{M}_i}=0}\right] = n - |\mathcal{I}|$. Next, we provide an estimate for $\Pr\left(\mathcal{G}_i^c\right)$ 
\[
\nonumber
\Pr\left(\mathcal{G}_i^c\right) &= \Pr\left(\inner{\hat{\theta}-\mu}{Z_{\bar{U_i},i}-\mu}\geq\beta/n\right)\\
&=\EE\left[ \Pr\left(\inner{\hat{\theta}-\mu}{Z_{\bar{U_i},i}-\mu}\geq\beta/n\bigg| \hat \theta,U_i\right)\right].
\]
Since conditioned on $U_i$ and $\hat{\theta}$, $Z_{\bar U_i,i}\sim \Dist$ and $\Dist$ is a product measure, using \cref{lem:indep-sample-corr}, we have
\[
\nonumber
\Pr\left(\inner{\hat{\theta}-\mu}{Z_{\bar{U_i},i}-\mu}\geq\beta/n\right) \leq O\left(\frac{1}{n^2}\right).
\]
Also, by the well-known inequality, $\binaryentr{x}\leq -x\log(x)+x$ for $x\in [0,1]$, we have
\[
\nonumber
\mathrm{H}_b\left(\Pr\left(\mathcal{G}_i^c\right)\right)  \leq O\left(\frac{\log(n)}{n^2}\right).
\]
Therefore, using this estimates to simplify \cref{eq:iscmi-simpified}, we obtain
\[
\nonumber
\sum_{i=1}^{n} \entr{U_i \big| \hat{\theta},Z_{0,i},Z_{1,i}} \leq n - \EE[|\mathcal{I}|] + O\left(\frac{\log(n)}{n}\right).
\]
Plugging this upper bound into \cref{eq:expansion-iscmi},
\[
\nonumber
\sum_{i=1}^{n}\minf{\hat{\theta},U_i \big| Z_{0,i},Z_{1,i}} &= n - \sum_{i=1}^{n} \entr{U_i \big| \hat{\theta},Z_{0,i},Z_{1,i}}\\
&\geq \EE[|\mathcal{I}|] - O\left(\frac{\log(n)}{n}\right).
\]
Finally, we use \cref{cor:cardI-scvx}, to conclude that 
\[
\nonumber
\sum_{i=1}^{n}\minf{\hat{\theta},U_i \big| Z_{0,i},Z_{1,i}} &\geq \EE[|\mathcal{I}|] - O\left(\frac{\log(n)}{n}\right)\\
&\geq \Omega\left(\frac{1}{\epsilon}\right) - O\left(\frac{\log(n)}{n}\right)\\
&\geq\Omega\left(\frac{1}{\epsilon}\right),
\]
where the last step follows since the minimum number of samples to $\epsilon$-learn $\mathcal{P}^{(d)}_{\text{scvx}}$ is $n\geq \Omega(1/\epsilon)$.

The proof of the CLB subclass of SCOs is the same: using the same techniques we can lower bound the ISCMI by $\EE[|\mathcal{I}|]$ and then by \cref{cor:cardI-cvx} the result follows. We don't repeat it here.

\end{document}